\documentclass[letterpaper, 10 pt, conference]{ieeeconf}  
\IEEEoverridecommandlockouts                              

\usepackage{graphics} 
\usepackage{subfig} 
\usepackage{wrapfig}

\usepackage{amsthm} 
\usepackage{amsmath,amssymb}
\usepackage{graphicx}
\usepackage{tabularx}
\usepackage{multirow, makecell}
\usepackage{diagbox}
\usepackage{slashbox}
\usepackage{rotating}
\usepackage{cite}

\usepackage{url}
\usepackage{bm}
\usepackage[table]{xcolor}
\usepackage{cancel} 
\usepackage{hyperref}
\usepackage{siunitx} 
\usepackage{booktabs}
\usepackage{cleveref}
\usepackage{mathtools} 

\usepackage{comment}
\usepackage{lipsum}
\usepackage{caption}

\let\labelindent\relax 
\usepackage{enumitem} 

\usepackage[ruled,vlined]{algorithm2e}

\usepackage[nolist, nohyperlinks]{acronym}
\acrodef{MPC}[MPC]{Model Predictive Control}
\acrodef{QP}[QP]{Quadratic Program}
\acrodef{CBF}[CBF]{Control Barrier Function}

\newcommand{\vx}{{\boldsymbol x}}
\newcommand{\vu}{{\boldsymbol u}}

\newcommand{\vp}{{\boldsymbol p}} 
\newcommand{\vv}{{\boldsymbol v}} 

\newcommand{\lieder}{L}
\newcommand{\StateSpace}{\mathcal{X}}

\newcommand{\RealSpace}{\mathbb{R}}

\newcommand{\Rn}{\mathbb{R}^{n}}

\newcommand{\calC}{\mathcal{C}} 
\newcommand{\calK}{\mathcal{K}} 

\newcommand{\calX}{\mathcal{X}} 

\newcommand{\calU}{\mathcal{U}} 

\newcommand{\obsidx}{j}

\newtheorem{definition}{Definition}
\newtheorem{theorem}{Theorem}
\newtheorem{lemma}{Lemma}
\newtheorem{proposition}{Proposition}
\newtheorem{corollary}{Corollary}
\newtheorem{remark}{Remark}

\theoremstyle{definition}
\newtheorem{problem}{Problem}

\theoremstyle{definition}

\theoremstyle{definition}

\theoremstyle{definition}
\newtheorem{assumption}{Assumption}

\makeatother

\DeclareCaptionFont{mysize}{\fontsize{8}{9.6}\selectfont}
\title{\LARGE \bf Beyond Collision Cones: Dynamic Obstacle Avoidance for Nonholonomic Robots via Dynamic Parabolic Control Barrier Functions}
 
\author{Hun Kuk Park$^*$, Taekyung Kim$^*$ and Dimitra Panagou
\thanks{$^{*}$These authors contributed equally to this work}
\thanks{The authors are with the Department of Robotics, University of Michigan, Ann Arbor, MI 48109 USA {\tt\footnotesize parkcart@umich.edu, taekyung@umich.edu, dpanagou@umich.edu} }
\thanks{Dimitra Panagou is also with the Department of Aerospace Engineering, University of Michigan, Ann Arbor, MI, 48109, USA}%
}

\begin{document}
\maketitle
\thispagestyle{empty}
\pagestyle{empty}
\begin{abstract}
Control Barrier Functions (CBFs) are a powerful tool for ensuring the safety of autonomous systems, yet applying them to nonholonomic robots in cluttered, dynamic environments remains an open challenge. State-of-the-art methods often rely on collision-cone or velocity-obstacle constraints which, by only considering the angle of the relative velocity, are inherently conservative and can render the CBF-based quadratic program infeasible, particularly in dense scenarios. To address this issue, we propose a Dynamic Parabolic Control Barrier Function (DPCBF) that defines the safe set using a parabolic boundary. The parabola's vertex and curvature dynamically adapt based on both the distance to an obstacle and the magnitude of the relative velocity, creating a less restrictive safety constraint. We prove that the proposed DPCBF is valid for a kinematic bicycle model subject to input constraints. Extensive comparative simulations demonstrate that our DPCBF-based controller significantly enhances navigation success rates and QP feasibility compared to baseline methods. Our approach successfully navigates through dense environments with up to 100 dynamic obstacles, scenarios where collision cone-based methods fail due to infeasibility. \href{https://www.taekyung.me/dpcbf}{\textcolor{red}{[Project Page]}}\footnote{Project page: \href{https://www.taekyung.me/dpcbf}{https://www.taekyung.me/dpcbf}} \href{https://github.com/tkkim-robot/safe_control/tree/main/dynamic_env}{\textcolor{red}{[Code]}} \href{https://youtu.be/57qgoe7YJao}{\textcolor{red}{[Video]}} 
\end{abstract}

\section{INTRODUCTION}

Ensuring safety is a fundamental challenge for autonomous systems, particularly nonholonomic robots and autonomous vehicles operating in dynamic and cluttered environments. Control Barrier Functions~(CBFs) have emerged as a powerful tool for enforcing safety constraints in real-time, formulated within a Quadratic Program~(QP)~\cite{ames_control_2017} or with Model Predictive Control~(MPC)~\cite{zeng_safetycritical_2021}. Their effectiveness has led to widespread adoption in applications from robotic navigation~\cite{kim_visibilityaware_2025} to multi-agent coordination~\cite{zhang_gcbf_2025}. %

Collision avoidance can be encoded through a distance-based CBF, which defines the safe set based on the Euclidean distance to an obstacle. To incorporate the relative velocity between the robot and the obstacle, one can employ a High-Order CBF~(HOCBF)~\cite{xiao_control_2019}. However, it requires all control inputs to appear in the CBF condition, which makes it difficult to be applied to systems with inputs of different relative degrees~\cite{garg_advances_2024}.

Recent work addresses dynamic obstacles within the CBF framework by leveraging velocity-obstacle~(VO) constraints~\cite{fiorini_motion_1998}, also referred to as collision cones in other literature~\cite{chakravarthy_obstacle_1998}. These methods define the unsafe set as a collision cone in the relative-velocity space and constrain the relative velocity to lie outside a fixed cone~\cite{tayal_collision_2024, huang_dynamic_2025}. This approach has been successfully applied to various systems, including the kinematic bicycle model, by showing that the constraint has relative degree one with respect to all control inputs. Despite their advantages for dynamic obstacle avoidance, cone-based and VO-based methods exhibit fundamental conservatism. Because the safety constraint depends only on the heading angle of the relative velocity, the robot is prohibited from moving toward the obstacle, regardless of their distance or relative speed. This rigidity can induce immediate QP infeasibility when the initial relative velocity lies within a collision cone, or in dense environments where the union of multiple cones removes all feasible control inputs, even when sufficient collision-free space exists (see Fig.~\ref{fig:dpcbf_intro}a).

\begin{figure}[t]
    \centering
    \includegraphics[width=0.99\linewidth]{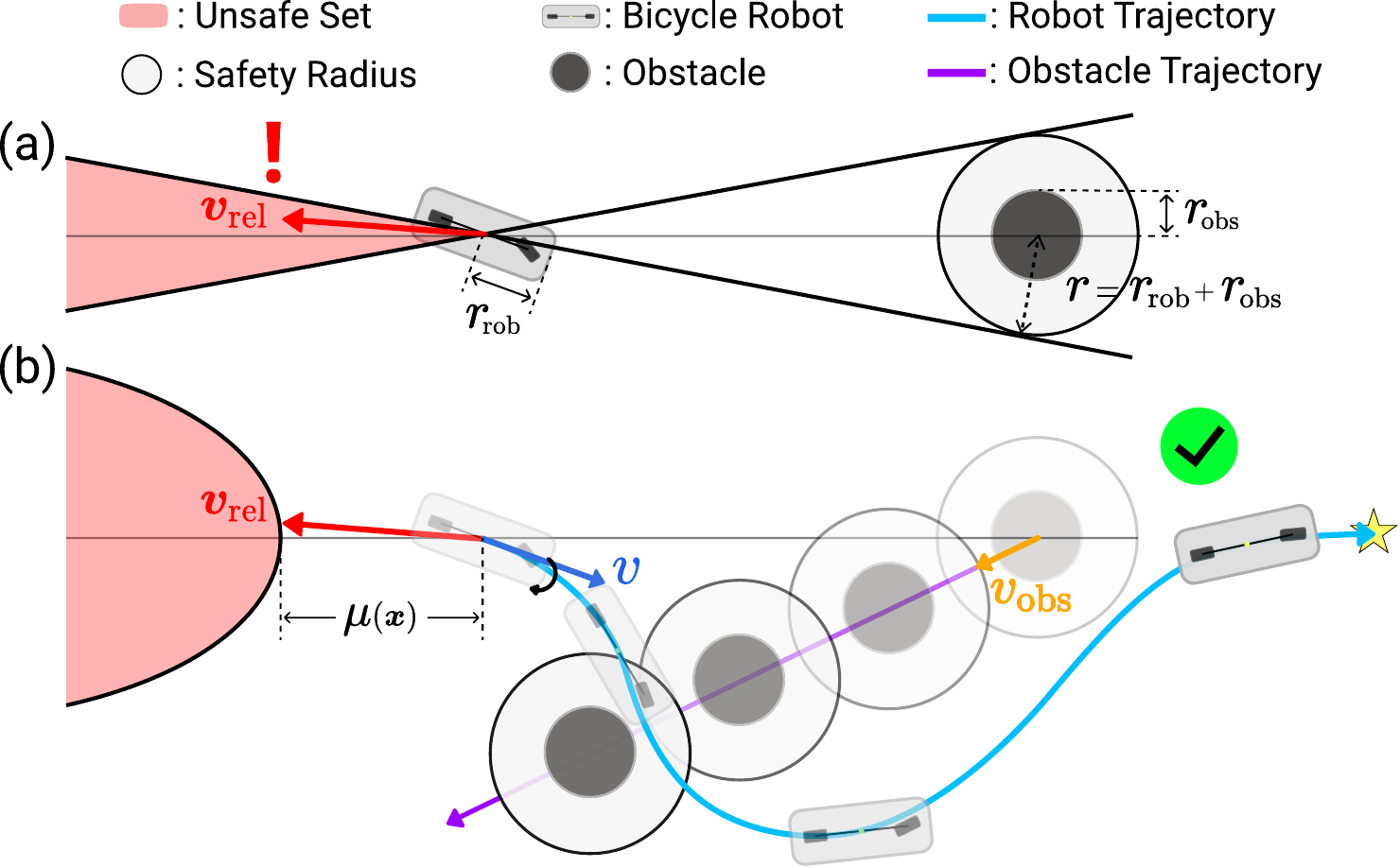}
    \caption{Illustrative comparison of two CBF mechanisms in dynamic obstacle avoidance scenarios. (a) Since the Collision Cone CBF~(C3BF) evaluates only the heading of the relative velocity, it may classify the robot as unsafe regardless of its actual distance from the obstacle. (b) Our Dynamic Parabolic CBF~(DPCBF) establishes a more flexible safety condition by evaluating both relative position and the magnitude of relative velocity, which avoids unnecessary restrictions when clearance is large. As shown in (b), the parabola's vertex shifts away from the robot's origin by $\mu(\vx)$. This relaxes the safety constraint, allowing for less restrictive movements that approach the boundary of the unsafe set while remaining provably safe.}
    \label{fig:dpcbf_intro}
\end{figure} %

This paper introduces a \textbf{\emph{Dynamic Parabolic Control Barrier Function~(DPCBF)}} that explicitly incorporates both clearance and the magnitude of the relative velocity. Instead of a fixed cone, we define a state-dependent parabolic safety boundary whose curvature and vertex adapt with distance and relative velocity (see Fig.~\ref{fig:dpcbf_intro}b). This design yields less conservative safety constraints, improving the feasibility of the CBF-based controller in cluttered, dynamic environments. The main contributions of this work are:

\begin{itemize}
\item We propose a DPCBF for nonholonomic robots in dynamic obstacle avoidance tasks, which dynamically shapes the safety boundary to provide less conservative safety margins by adapting to distance and relative velocity.
\item We prove that DPCBF is valid for the kinematic bicycle model under input constraints.
\item We show extensive simulation results in dense, dynamic environments, demonstrating higher feasibility and success rates, and lower control intervention, compared to state-of-the-art CBF methods.
\end{itemize}
\section{{PRELIMINARIES} \label{sec:preliminaries}}
\subsection{Control Barrier Functions}

Consider a continuous-time, control-affine system:
\begin{equation}\label{eq:cbf_system}
  \dot{\vx} = f(\vx) + g(\vx) \vu,
\end{equation}
where $\vx \in \calX \subset \mathbb R^{n}$ is the state and $\vu\in \calU \subset\mathbb R^{m}$ is the control input, with $\calU$ representing the admissible control set for System~\eqref{eq:cbf_system}. The functions $f:\calX \to\mathbb R^{n}$ and $g:\calU \to\mathbb R^{n\times m}$ are both assumed to be locally Lipschitz continuous.

Let $h:\mathbb R^{n} \to\mathbb R$ be a continuously differentiable function. We define
\begin{subequations}
\begin{align}\label{eq:def_safe_set}
    \calC \coloneqq \{\vx \in \mathbb R^{n} \mid h(\vx) \geq 0 \},\\
  \label{eq:def_safe_set_boundary}\partial \calC \coloneqq \{\vx \in \mathbb R^{n} \mid h(\vx) = 0 \},\\
  \text{Int}(\calC) \coloneqq \{\vx \in \mathbb R^{n} \mid h(\vx) > 0 \},
\end{align}
\end{subequations}
where $\calC$ is referred to as the \emph{safe set}. 

\begin{definition}[Forward Invariance]\label{def:fwd_inv}
A closed set $\calC \subset \Rn$ is \emph{forward invariant} for System~\eqref{eq:cbf_system} under a state-feedback control law $\vu=\pi(\vx)$ if the solution $\vx(t)$ of the closed-loop system $\dot{\vx}(t) = f(\vx(t)) + g(\vx(t)) \pi(\vx(t))$ for every initial state $\vx(0) \in \cal{C}$ satisfies $\vx(t)\in\calC, \forall t \ge 0$.
\end{definition} %

\begin{definition}[Control Barrier Function~\cite{ames_control_2017}]\label{def:cbf_basic}
Given the set $\calC$ defined by \eqref{eq:def_safe_set}, the function $h$ is a \emph{CBF} for System~\eqref{eq:cbf_system} if there exists an extended class $\calK_{\infty}$
function $\alpha(\cdot)$ such that
\begin{equation}\label{cond:cbf_basic}
  \sup_{\vu \in \calU} \bigl[\underbrace{\lieder_fh(\vx)+\lieder_gh(\vx)\vu}_{\dot h(\vx, \vu)}\bigr]
  \geq -\alpha(h(\vx))
  \quad\forall \vx \in \calC . \end{equation} 
  \end{definition}
We denote $\lieder_f h$ and $\lieder_g h$ as the Lie derivatives of the function $h$ with respect to $f$ and $g$. %

\begin{lemma}\label{lem:def_cbf_controller}\textup{(\cite[Theorem 1]{ahmadi_safe_2019})}
Let $h$ satisfy the CBF condition~\eqref{cond:cbf_basic} and define
\begin{equation}
    K_{\textup{cbf}}(\vx) 
    \coloneqq \Bigl\{\vu \in \calU \bigm|
          \lieder_fh(\vx) + \lieder_gh(\vx) \vu
          \ge - \alpha \bigl(h(\vx) \bigr) \Bigr \}.
\end{equation}
Then, any Lipschitz continuous feedback controller $\vu=\pi(\vx)\in K_{\textup{cbf}}(\vx)$ renders
$\mathcal C$ forward invariant for System~\eqref{eq:cbf_system}.
\end{lemma}

To enforce that trajectories of \eqref{eq:cbf_system} remain in $\calC$ \eqref{eq:def_safe_set}, we solve the following Quadratic Program with CBF constraint~(CBF-QP):
\begin{align}
  \vu^{\star}(\vx)
  =&
  \arg\min_{\vu\in\mathcal U}
      \|\vu-\vu_{\textup{ref}}(\vx)\|_{2}^{2} \label{eq:cbf_qp}\\
  \text{s.t. }&
      \lieder_f h(\vx)+\lieder_g h(\vx) \vu
      \ge -\alpha \bigl(h(\vx)\bigr). \nonumber
\end{align}
Note, inputs are bounded:  $\calU \not = \mathbb R^{m}$. By \autoref{lem:def_cbf_controller}, if $h$ is a CBF, applying $\vu=\vu^\star(\vx)$ guarantees the state in the safe set~$\calC$ for all time. %
\subsection{Bicycle Model}

In this paper, we consider a robot modeled by the kinematic bicycle model~\cite{polack_kinematic_2017,he_rulebased_2021} (see Fig. \ref{fig:dpcbf_bicycle_schemetic}). The state is $\vx=[x, y, \theta, v]^{\top}$, where $(x,y)$ denotes the position of the vehicle’s center of mass (CoM), $\theta$ is the heading angle, and $v$ is the forward velocity. The control inputs are longitudinal acceleration $a$ and the forward-wheel steering angle $\delta$. Let $\ell_f$ and $\ell_r$ denote the distances from the CoM to the front and rear axles, respectively, and define the slip angle $\beta=\tan^{-1} \bigl( \tan(\delta) \, \ell_r/(\ell_f+\ell_r) \bigr)$. 

\begin{figure}[t]
    \centering
    \includegraphics[width=0.99\linewidth]{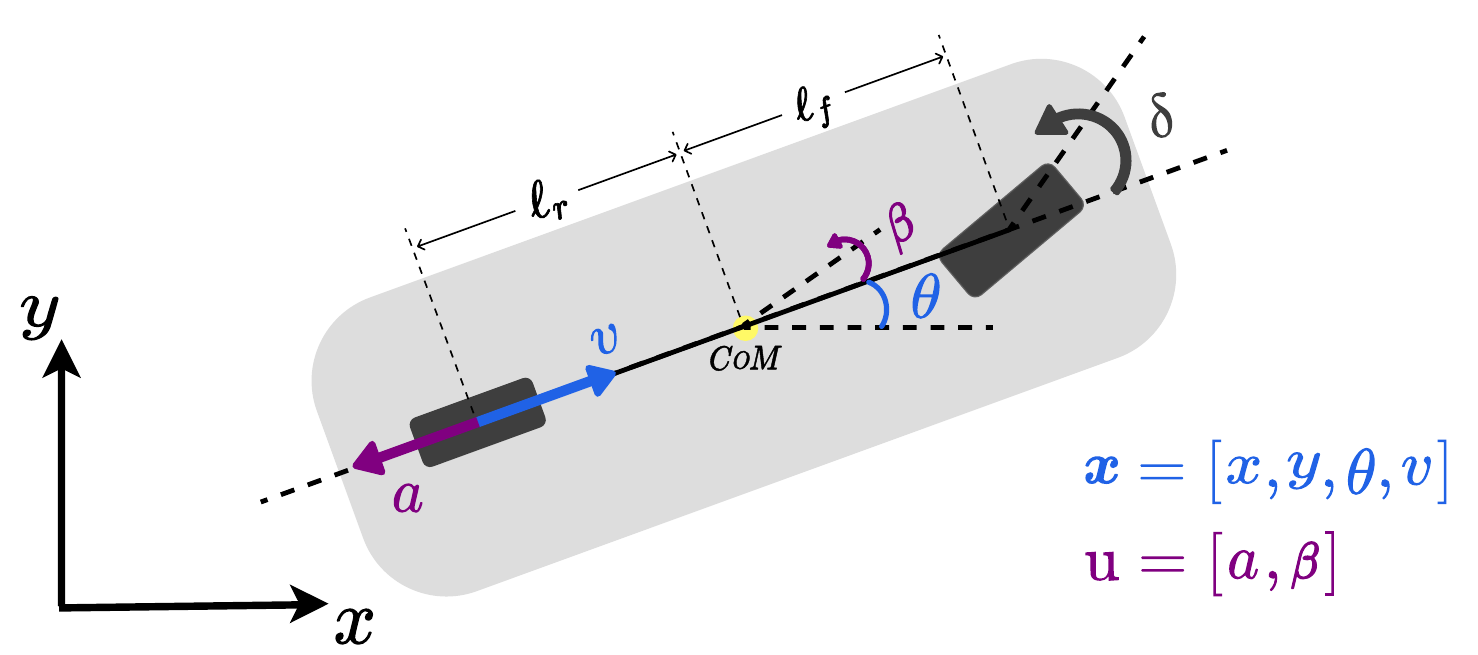}
    \caption{Schematic of the kinematic bicycle model. The robot's state is defined by its Center of Mass (CoM) position $(x,y)$, heading angle $\theta$, and forward velocity $v$. The distances from the CoM to the front and rear axles are $\ell_{f}$ and $\ell_{r}$, respectively. The front-wheel steering angle is $\delta$, and $\beta$ is the resulting vehicle slip angle.}
    \label{fig:dpcbf_bicycle_schemetic}
\vspace{-5pt}
\end{figure}

To model the kinematic bicycle as a control affine system as in \eqref{eq:cbf_system}, we consider that the slip angle~$\beta$ is small, i.e., $\sin \beta \approx \beta$. Then, the dynamics equation follows~\cite{polack_kinematic_2017}
\begin{equation}\label{eq:bicycle_simplified_sys}
\underbrace{\begin{bmatrix}
\dot x\\[2pt]
\dot y\\[2pt]
\dot\theta\\[2pt]
\dot v
\end{bmatrix}}_{\dot\vx}
=
\underbrace{\begin{bmatrix}
v\cos\theta\\[2pt]
v\sin\theta\\[2pt]
0\\[2pt]
0
\end{bmatrix}}_{f(\vx)}
+
\underbrace{\begin{bmatrix}
0 & -v\sin\theta\\[2pt]
0 & \,v\cos\theta\\[2pt]
0 & \dfrac{v}{\ell_r}\\[6pt]
1 & 0
\end{bmatrix}}_{g(\vx)}
\underbrace{\begin{bmatrix}
a\\[2pt]
\beta
\end{bmatrix}}_{\vu},
\end{equation}
where the inputs are now $\vu=[a, \beta]^{\top}$.

\subsection{Obstacle Model}

We model a scenario with multiple dynamic obstacles. The state of the $\obsidx$-th dynamic obstacle, where $\obsidx \in \{1, \ldots, N_{\textup{obs}}\}$, is represented by
\begin{equation}
\vx_{\textup{obs}}^{\obsidx}
  = [x_{\textup{obs}}^{\obsidx}, y_{\textup{obs}}^{\obsidx}, \theta_{\textup{obs}}^{\obsidx} ,v_{\textup{obs}}^{\obsidx}]^{\top},
\end{equation}
where $x_{\textup{obs}}^{\obsidx},y_{\textup{obs}}^{\obsidx}$ denote the obstacle’s center position, 
$\theta_{\textup{obs}}^{\obsidx}$ its heading angle, and $v_{\textup{obs}}^{\obsidx}$ its forward speed.

The dynamics of the $\obsidx$-th obstacle is described by a unicycle model with constant velocity: $\dot{x}_{\textup{obs}}^{\obsidx} = v_{\textup{obs}}^{\obsidx} \cos \theta_{\textup{obs}}^{\obsidx}$ and $ \dot{y}_{\textup{obs}}^{\obsidx}=v_{\textup{obs}}^{\obsidx}\sin \theta_{\textup{obs}}^{\obsidx}$.
We assume that the obstacle states are fully observable. For the remainder of the paper, we omit the superscript~$\obsidx$ and describe the CBF constraint for each obstacle for notational simplicity. 
\subsection{Distance-Based CBF}

A distance-based CBF is a collision avoidance formulation based solely on the Euclidean distance. Let $r_{\textup{rob}}>0$ and $r_{\textup{obs}}>0$ denote conservative safety radii that over-approximate the robot and obstacle geometries, and define $r\coloneqq r_{\textup{rob}}+r_{\textup{obs}}$. Then, with the robot position $\vp=[x, y]^{\top}$ and obstacle position $\vp_{\textup{obs}}=[x_{\textup{obs}}, y_{\textup{obs}}]^{\top}$, a distance-based safety constraint function is:
\begin{equation}\label{eq:dist_cbf}
h_{\textup{dist}}(\vx) = \|\vp-\vp_{\textup{obs}}\|^{2}-r^{2}.
\end{equation}
Because it only considers distance, this barrier is not, in general, a CBF except for simple systems in which the control inputs directly affect the velocity. Its myopic nature makes it particularly unsuitable for systems with nonholonomic constraints~\cite{tayal_collision_2024}.
\subsection{Collision Cone CBF \label{sec:c3bf}}

\begin{figure}[t]
    \centering
    \includegraphics[width=0.99\linewidth]{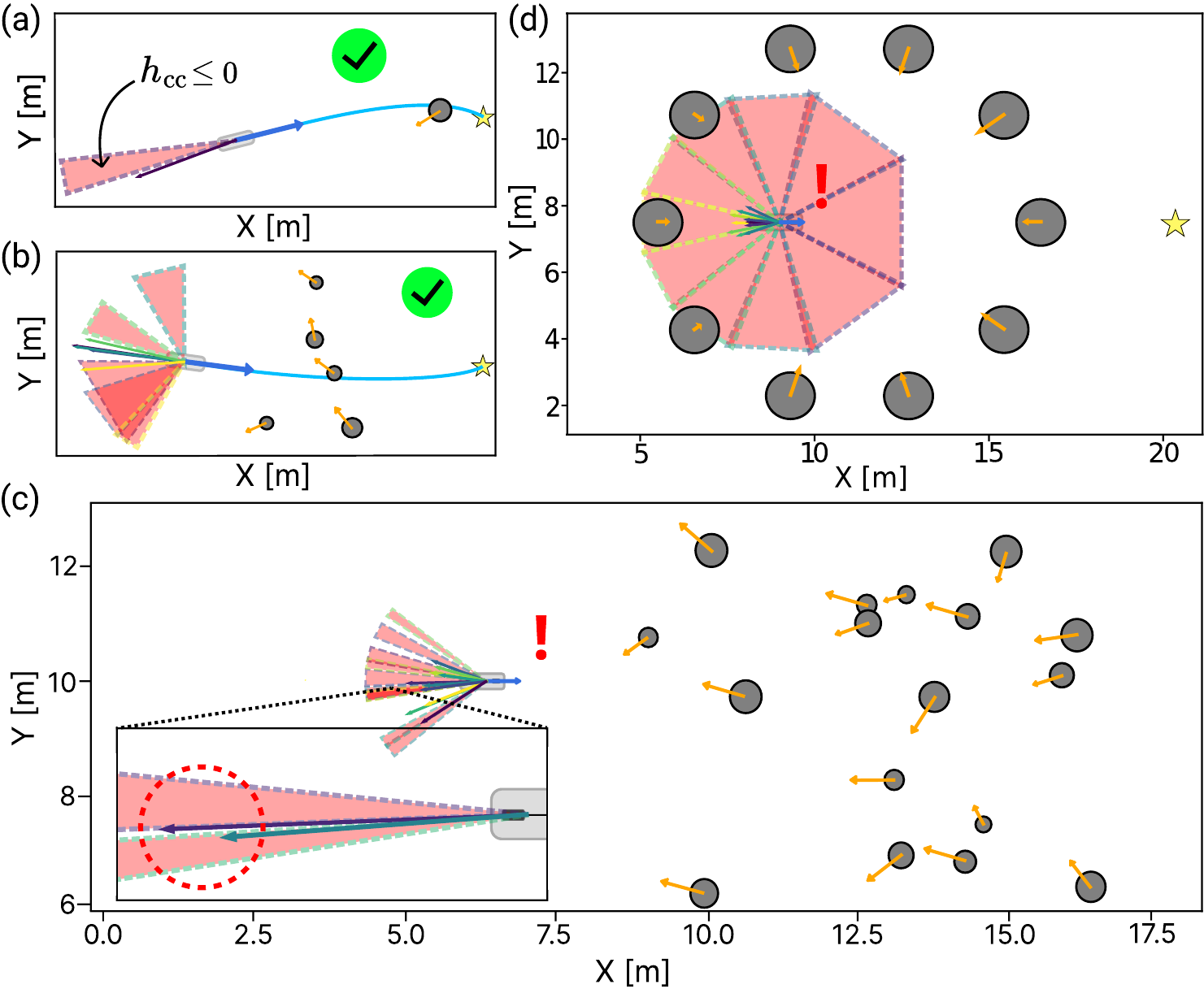}
    \caption{A closer look at collision cone-based CBF. (a) Single obstacle: if $h_{\textup{cc}}(\vx(t_0)) \geq 0$, the CBF constraint keeps $h_{\textup{cc}}(\vx(t)) \ge 0$ for all $t \ge t_0$. 
(b) Five obstacles: if $h_{\obsidx, \textup{cc}}(\vx(t)) \geq 0$ for all $\obsidx$-th obstacle and the CBF-QP is feasible at $t$, safety is at least maintained at a given time $t$.
(c) Even with $h_{\obsidx,\textup
{cc}}(\vx(t_0)) \ge 0$ for all $\obsidx$-th obstacle, cone intersections can leave no admissible relative velocity direction, leading the CBF-QP infeasible. (d) Given the initial configuration where the robot is surrounded by the union of the collision cones, there is no feasible solution to the CBF-QP even though a large collision-free area exists nearby.}
\label{fig:c3bf_cases} 
\vspace{-5pt}
\end{figure} %

The Collision Cone CBF (C3BF)~\cite{tayal_collision_2024} was recently proposed for dynamic obstacle avoidance and constructs the CBF with the Velocity Obstacle (VO)~\cite{fiorini_motion_1998} constraint. Given the relative position~$\vp_{\textup{rel}}$ and relative velocity~$\vv_{\textup{rel}}$, a conservative circle of radius $r$ is placed around the obstacle center, and the collision cone is formed by the pair of tangents from the robot's center to the circle (see Fig.~\ref{fig:dpcbf_intro}a). To implement the outside-of-cone constraint, C3BF defines
\begin{equation} \label{eq:c3bf_candidate}
    h_{\textup{cc}}(\vx) = \langle\vp_{\textup{rel}}, \vv_{\textup{rel}}\rangle + \|\vp_{\textup{rel}}\|\|\vv_{\textup{rel}}\|\,\cos \phi
\end{equation}
where $\phi$ is half the cone angle, and $\cos\phi= \frac{\sqrt{\|\vp_{\textup{rel}}\|^{2}-r^{2}}}{\|\vp_{\textup{rel}}\|}$. 
As illustrated in Fig.~\ref{fig:c3bf_cases}a, the unsafe set is the collision cone, i.e., $\{\vx \mid  h_{\textup{cc}}(\vx) < 0\}$. If the relative velocity at time $t_0$ lies outside this cone, the CBF constraint enforces the control input such that the relative velocity remains outside for all future time, thereby avoiding collision. This collision-cone approach is designed for moving obstacle avoidance tasks and it is also shown to be applicable to the kinematic bicycle model~\eqref{eq:bicycle_simplified_sys}~\cite{thontepu_collision_2023}. Recently, it also has been extended to navigation tasks of quadrotors~\cite{tayal_control_2024}, ground mobile robots, and autonomous vehicles~\cite{thontepu_collision_2023}.

However, the mechanism itself poses conversely fundamental limitations of the C3BF. Since the CBF only monitors the relative velocity's angle with respect to the collision cone, the robot cannot ever drive towards the unsafe set, no matter how far away from those unsafe sets and how small the velocities are, the resulting behavior is \textbf{\emph{extremely conservative}}. In addition, if the initial relative velocity lies inside of the collision cone whenever the controller just gets initiated, the problem becomes immediately infeasible, even though there is a large free space in between the robot and the obstacle (see Fig.~\ref{fig:c3bf_cases}c). Furthermore, this problem is more prominent in multi-obstacle cases as shown in Fig.~\ref{fig:c3bf_cases}d. If the robot is surrounded by obstacles, the union of each cone shrinks the set of admissible relative velocity directions, making it easily infeasible in dense environments.  %

\section{DYNAMIC PARABOLIC CBF\label{sec:DPCBF_main}}

In this paper, we present a novel CBF formulation for dynamic-obstacle collision avoidance tasks. Existing methods either do not provide safety guarantees for moving obstacles or are prone to infeasibility when multiple obstacles are nearby. Accordingly, we focus on improving the key criteria: (i) guaranteeing safety for dynamic obstacles under input constraints, and (ii) improving the feasibility of the resulting CBF-based controller.

At a high level, we construct a safety constraint that explicitly accounts for the \textbf{\emph{magnitude of the relative velocity}}. Unlike C3BF, which relies on a fixed cone and only evaluates the heading of the relative velocity, our approach allows for a less restrictive safety condition. This distinction is crucial, as it permits the robot to safely move toward an obstacle when the relative velocity is low and clearance is large. We introduce a geometric strategy inspired by finite-time velocity obstacle formulations~\cite{guy_clearpath_2009}, in particular the truncated cone construction and parabolic approximation of the safe set boundary~\cite{kim_study_2016, samavati_optimal_2017}.
\subsection{DPCBF Formulation}

Consider a robot modeled as System~\eqref{eq:bicycle_simplified_sys} navigating in an environment with dynamic obstacles.

\begin{figure}[t]
    \centering
    \includegraphics[width=0.99\linewidth]{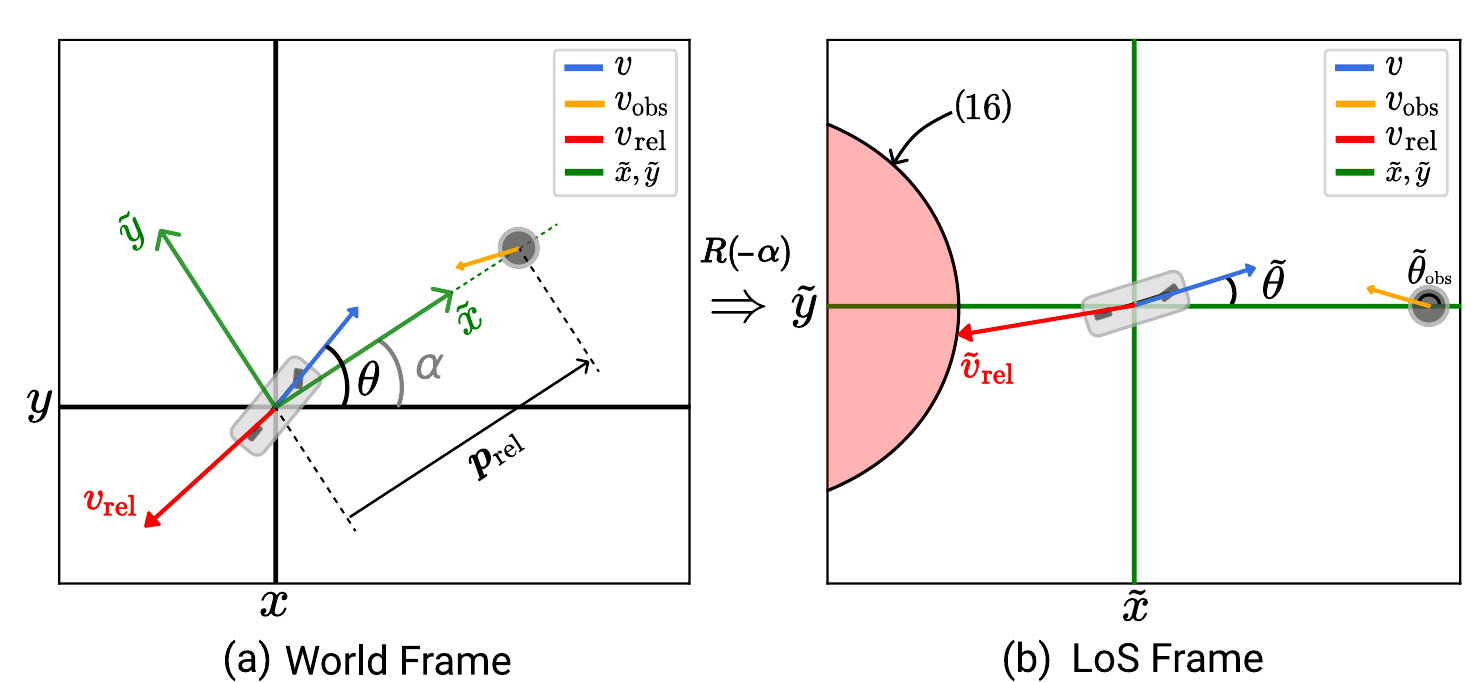}
    \caption{Visualization of the global world frame $(x,y)$ and the rotated Line-of-Sight (LoS) frame $(\tilde x,\tilde y)$ used in our formulation. By rotating the coordinates by an angle $\alpha$~\eqref{eq:dpcbf_alpha_def}, the $\tilde x$-axis of the LoS frame is aligned with the vector from the robot to the obstacle, $\vp_{\textup{rel}}$. This transformation simplifies the definition of the parabolic safety boundary~\eqref{eq:parabolic_region}, allowing its position and curvature to adapt online based on the relative velocity components in this new frame.}
    \label{fig:dpcbf_frames}
\vspace{-5pt}
\end{figure}

\subsubsection{Relative Coordinates}
Define the relative position and velocity between the robot and the obstacle:
\begin{subequations}
\begin{align}
\vp_\textup{rel} &= \begin{bmatrix}
    p_{\textup{rel},x}\\
    p_{\textup{rel},y}
\end{bmatrix}= \begin{bmatrix}
    x_{\textup{obs}} - x\\
    y_{\textup{obs}} - y
\end{bmatrix} \in \RealSpace^{2}, \\
    \vv_\textup{rel} &= \begin{bmatrix}
    v_{\textup{rel},x}\\
    v_{\textup{rel},y}
\end{bmatrix}= \begin{bmatrix}
    v_{\textup{obs}}\cos\theta_{\textup{obs}} - v\cos\theta\\
    v_{\textup{obs}}\sin\theta_{\textup{obs}} - v\sin\theta
\end{bmatrix} \in \RealSpace^{2}
\end{align}
\end{subequations}
with norms $\|\vp_{\textup{rel}}\| = \sqrt{p_{\textup{rel},x}^2 + p_{\textup{rel},y}^2}$, $\|\vv_{\textup{rel}}\| = \sqrt{v_{\textup{rel},x}^2 + v_{\textup{rel},y}^2}$. Then, we rotate the coordinates to align with the line connecting the robot and the obstacle (see Fig. \ref{fig:dpcbf_frames}a). Denote the angle between the global \emph{x}-axis and this new \emph{x}-axis as the rotation angle:
\begin{subequations}
\begin{align}\label{eq:dpcbf_alpha_def}
      \alpha &= \operatorname{atan2}(p_{\textup{rel},y}, p_{\textup{rel},x}), \\
\label{eq:dpcbf_rotation}
      \mathbf R(-\alpha)
  &=
  \begin{bmatrix}
    \cos\alpha & \sin\alpha \\
    -\sin\alpha & \cos\alpha
  \end{bmatrix}
  \in SO(2).
\end{align}
\end{subequations}
We refer to the rotated frame defined by the rotation matrix $\mathbf{R}$ as Line-of-Sight (LoS) frame throughout this paper (see Fig. \ref{fig:dpcbf_frames}b). Finally, we define the relative velocity in the LoS frame:
\begin{equation}\label{eq:rot_rel_vel}
\tilde\vv_{\textup{rel}}=
    \begin{bmatrix}
\tilde v_{\textup{rel},x} \\
\tilde v_{\textup{rel},y}
\end{bmatrix}
=
 \mathbf R(-\alpha)
\begin{bmatrix}
v_{\textup{rel},x} \\
v_{\textup{rel},y}
\end{bmatrix}.
\end{equation}

\subsubsection{CBF Formulation and Design Maps}
Let $r \in \RealSpace$ be the combined robot-obstacle radius defined in \eqref{eq:dist_cbf}. Then,
\begin{equation}\label{eq:dist_clearance}
    d(\vx) = \sqrt{\|\vp_{\textup{rel}}\|^2 - r^2}.
\end{equation}

We introduce tunable parameters $k_\lambda,k_\mu>0$, and the following functions:
\begin{equation}\label{eq:dpcbf_func_def}
    \lambda(\vx) = k_\lambda \frac{d(\vx)}{\|\vv_{\textup{rel}}\|}, 
\quad 
\mu(\vx) = k_\mu \, d(\vx).
\end{equation}
Here $\lambda: \StateSpace \rightarrow \RealSpace$ adjusts the curvature of the parabola, and $\mu: \StateSpace \rightarrow \RealSpace$ shifts the parabola forward by the safe distance margin. 

We propose the Dynamic Parabolic CBF~(DPCBF) as follows:
\begin{equation}\label{eq:dpcbf_def_simple}
    h(\vx) = \tilde v_{\textup{rel},x} 
       + \lambda(\vx) \tilde v_{\textup{rel},y}^2
       + \mu(\vx),
\end{equation}
where $\tilde v_{\textup{rel},x}$ and $\tilde v_{\textup{rel},y}$ are the components of the relative velocity between the robot and the obstacle in the LoS frame. Since $\lambda(\vx)$ and $\mu(\vx)$ depend on the relative distance and speed, the parabola representing the unsafe set dynamically adapts its shape online to the current situation. The CBF is now defined by measuring how close the endpoint of the rotated relative velocity~\eqref{eq:rot_rel_vel} is to a specific parabolic region in this new plane (see Fig. \ref{fig:dpcbf_functions}):
\begin{equation}\label{eq:parabolic_region}
    \tilde v_{\textup{rel},x} = 
       -\lambda(\vx) \tilde v_{\textup{rel},y}^2
       - \mu(\vx).
\end{equation}
This provides a significant advantage over cone-based methods. As illustrated in Fig.~\ref{fig:dpcbf_intro}b, the boundary of the unsafe set~\eqref{eq:parabolic_region} does not pass through the origin of the LoS-frame relative-velocity plane whenever the clearance to the obstacle, $d(\vx)$, is non-zero. This shift creates a feasible space where motion toward an obstacle is no longer treated immediately as unsafe. Instead, safety is now evaluated jointly using the current clearance~$d(\vx)$ and the relative velocity~$\vv_{\textup{rel}}$. We empirically show in \autoref{sec:results} that this design improves the feasibility of the DPCBF-based controller compared with prior work.

\begin{figure}[t]
    \centering
    \includegraphics[width=0.99\linewidth]{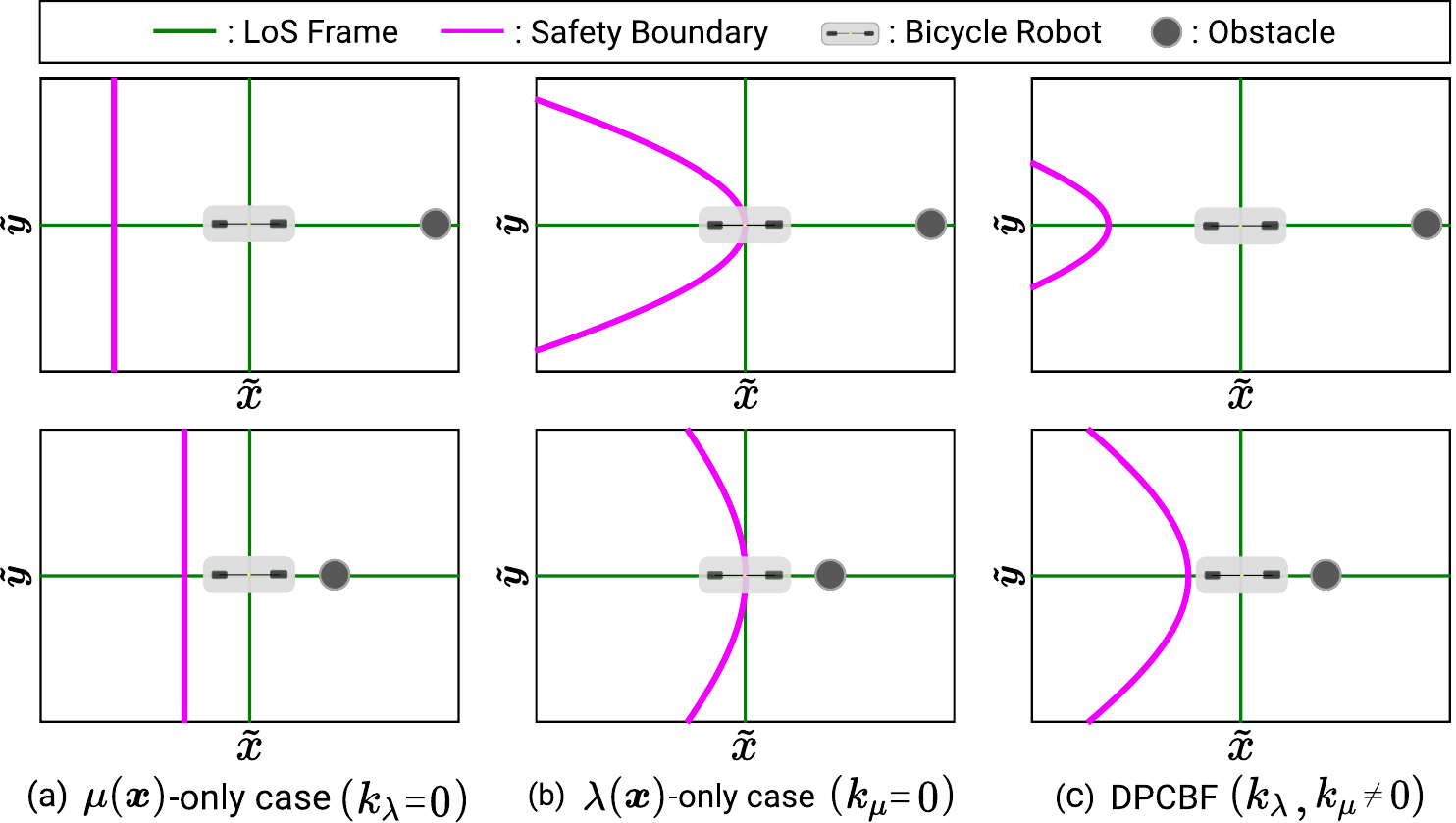}
    \caption{Three examples illustrating how \eqref{eq:parabolic_region} shapes the safety boundary. (a) When $k_{\lambda}=0$ (with $k_{\mu}$ active), as the obstacle gets closer, the parabola's vertex moves toward the robot, shrinking the safe region. (b) When $k_{\mu}=0$ (with $k_{\lambda}$ active), as the obstacle approaches or the relative velocity magnitude increases, the curvature of the parabola decrease, leading to a larger unsafe set. (c) In DPCBF, where $k_{\lambda}\neq 0$ and $k_{\mu}\neq 0$, both the vertex and the curvature of the parabola adapt dynamically.}
    \label{fig:dpcbf_functions}
\vspace{-5pt}
\end{figure}

\subsection{Validity of DPCBF}

To make DPCBF valid for System~\eqref{eq:bicycle_simplified_sys}, we require the following assumptions:
\begin{assumption}\label{assum:robot_speed_bound}
    The forward speed of System~\eqref{eq:bicycle_simplified_sys} is bounded by $v \in [v_{\min},\,v_{\max}]$, where $v_{\max}>v_{\min}>0$.
\end{assumption}
\autoref{assum:robot_speed_bound} is a required but mild assumption as in \cite{sadraddini_provably_2017, xiao_control_2019,haraldsen_safetycritical_2024}, since it excludes the degeneracy at $v=0$ in \eqref{cond:cbf_basic} where $\lieder_f h(\vx)=0$.
\begin{assumption}\label{assum:clearance}
    The admissible distance satisfies $p_{\max} \geq \|\vp_{\textup{rel}}\| \geq p_{\min} \coloneqq s\,r > 0$, where $s>1$ is a safety margin. Therefore, $d(\vx) = \sqrt{\|\vp_{\textup{rel}}\|^{2}-r^{2}} > d_{\min} \coloneqq \sqrt{p_{\min}^{2}-r^{2}} > 0$. The maximum distance to obstacle~$p_{\max}$ is determined by the finite sensing range.
\end{assumption}

To prove the proposed DPCBF is valid for System~\eqref{eq:bicycle_simplified_sys}, we show that for any state on safe set boundary~$\vx \in \partial\calC$, there exists an admissible control input $\vu \in \calU \coloneqq \{[a,\beta]^\top \mid|a|\leq a_{\max},|\beta|\leq~\beta_{\max}\}$ that satisfies the CBF condition \eqref{cond:cbf_basic}.

We first derive the corresponding terms for System~\eqref{eq:bicycle_simplified_sys}. Let $\Phi(\vx)$ denote the maximum control authority at $\vx$:
\begin{align}
    \Phi(\vx) \coloneqq& \sup_{\vu\in\mathcal U} \lieder_g h(\vx)\vu = \sup_{\vu\in\mathcal U} \begin{bmatrix} C^{a}(\vx) \\C^{\beta}(\vx)\end{bmatrix}^\top \vu \nonumber\\
          =& \big| C^{a}(\vx) \big|\,a_{\max} + \big| C^{\beta}(\vx) \big|\,\beta_{\max},
\end{align}
where $C^{a}(\vx)$ and $C^{\beta}(\vx)$ are the derived control terms:
\begin{align}
    &\big|C^{a}(\vx)\big| =  \Bigg| \underbrace{  \Bigl[
        -1
        + k_\lambda \frac{d(\vx)}{\|\vv_{\textup{rel}}\|^{3}}\, v_{\textup{obs}}\,\cos\tilde\theta_{\textup{obs}}\tilde v_{\textup{rel},y}^{2}
      \Bigr]}_{\coloneqq\eta_{a,\cos}(\vx)}\,
      \cos\tilde \theta \nonumber
  \\
    &+  \underbrace{\Bigl[
        k_\lambda \frac{d(\vx)}{\|\vv_{\textup{rel}}\|^{3}} v_{\textup{obs}}\sin\tilde\theta_{\textup{obs}}\tilde v_{\textup{rel},y}^{2}
          -2k_\lambda\frac{d(\vx)}{\|\vv_{\textup{rel}}\|}
            \tilde v_{\textup{rel},y}
      \Bigr]}_{\coloneqq\eta_{a,\sin}(\vx)} \sin\tilde \theta \nonumber \\
      &+ \underbrace{\Bigl[
        - k_\lambda \frac{d(\vx)}{\|\vv_{\textup{rel}}\|^{3}}
          v\,\tilde v_{\textup{rel},y}^{2}
      \Bigr]}_{\coloneqq\eta_{a,0}(\vx)} \Bigg| ,
\end{align}
and
\begin{equation}
    \bigl|C^{\beta}(\vx)\bigr|
  =
  \bigl|\eta_{\beta,\cos}(\vx)\cos\tilde\theta
          + \eta_{\beta,\sin}(\vx)\sin\tilde\theta\bigr|,
\end{equation}
where
\begin{subequations}
\begin{align}
&\eta_{\beta,\cos}(\vx) \coloneqq v\,\biggl[
          -\frac{\tilde v_{\textup{rel},y}}{\|\vp_{\textup{rel}}\|}
          + 2k_\lambda\,\frac{d(\vx)}{\|\vv_{\textup{rel}}\|}\,
            \frac{\tilde v_{\textup{rel},y}\tilde v_{\textup{rel},x}}{\|\vp_{\textup{rel}}\|} \nonumber \\
        &+  k_\lambda \frac{v}{\ell_{r}} \frac{d(\vx)}{\|\vv_{\textup{rel}}\|}\,
                \tilde v_{\textup{rel},y}
                \Bigl(
             \frac{v_{\textup{obs}}}{\|\vv_{\textup{rel}}\|^{3}}
            \sin\tilde\theta_{\textup{obs}}\,
                \tilde v_{\textup{rel},y}
            - 2 \Bigr)
          \biggr],\\
        &\eta_{\beta,\sin}(\vx) \coloneqq v\,\biggl[
            k_\lambda\,\frac{\|\vp_{\textup{rel}}\|}{d(\vx)}\,
              \frac{\tilde v_{\textup{rel},y}^{2}}{\|\vv_{\textup{rel}}\|}
          + k_\mu\,\frac{\|\vp_{\textup{rel}}\|}{d(\vx)} \nonumber\\
        &+ \frac{v}{\ell_{r}}\Bigl(
              1
            - k_\lambda\frac{d(\vx)}{\|\vv_{\textup{rel}}\|^{3}}\,
                v_{\textup{obs}}\cos\tilde\theta_{\textup{obs}}\,
                \tilde v_{\textup{rel},y}^{2}
            \Bigr)
          \biggr] .
\end{align}
\end{subequations}

Now, we aim to verify the following Nagumo's condition:
\begin{equation}\label{eq:nagumo_eq}
  \lieder_f h(\vx) + \Phi(\vx) \ge 0\qquad \forall \vx\in\partial\mathcal C,
\end{equation}
where
\begin{align}
    &\lieder_{f}h(\vx) = v\biggl[\biggl(
  -k_\lambda
     \frac{\|\vp_{\textup{rel}}\|}{d(\vx)}
     \frac{\tilde v_{\textup{rel},y}^2}{\|\vv_{\textup{rel}}\|}
     -k_\mu \frac{\|\vp_{\textup{rel}}\|}{d(\vx)} \biggr) \cos\tilde \theta\nonumber\\
  &\qquad \quad + \biggl(2k_\lambda
     \frac{\tilde v_{\textup{rel},y}}{\|\vv_{\textup{rel}}\|}
     \frac{d(\vx)}{\|\vp_{\textup{rel}}\|}\tilde v_{\textup{rel},x}
     -\frac{\tilde v_{\textup{rel},y}}{\|\vp_{\textup{rel}}\|}
  \biggr)\sin\tilde \theta \biggr] .
\end{align}

We partition the safe set boundary $\partial\mathcal C$ as $\partial\mathcal C_{1}$ and $\partial\mathcal C_{2}$, i.e., $\partial\mathcal C_{1} \cup \partial\mathcal C_{2} = \partial \calC$:
\begin{equation}\label{eq:dpcbf_split}
    \partial\mathcal C_{1}\coloneqq\{\vx \mid | \sin\tilde{\theta} | \ge \bar{s} \},\quad
  \partial\mathcal C_{2}\coloneqq\{\vx \mid | \sin\tilde{\theta} | < \bar{s} \},
\end{equation}
where $\bar s \coloneqq \frac{v_{\textup{obs}}}{v}\sin \tilde \theta_{\textup{obs}}\in[0,1)$. Therefore, we verify \eqref{eq:nagumo_eq} in these two sub-groups separately, for $i \in \{ 1, 2\}$, yielding:
 \allowdisplaybreaks
\begin{subequations}
\begin{align}\label{eq:dpcbf_boundary_split}
&\inf_{\vx\in\partial\mathcal C_i}\big[L_f h(\vx)+|C^{a}(\vx)|a_{\max}
+ |C^{\beta}(\vx)|\beta_{\max}\big]   \\ 
\geq & \underbrace{\inf_{\vx\in\partial\mathcal C_i} L_f h(\vx)}_{D_{i, \min}(k_{\lambda}, k_{\mu})} + \underbrace{\inf_{\vx\in\partial\mathcal C_i}\big[|C^{a}(\vx)|a_{\max} \big]}_{C^{a}_{i, \min}(k_{\lambda}, k_{\mu})} \nonumber\\
& \qquad\qquad\qquad\qquad\qquad + \underbrace{\inf_{\vx\in\partial\mathcal C_i}\big[|C^{\beta}(\vx)|\beta_{\max}\big]}_{C^{ \beta}_{i, \min}(k_{\lambda}, k_{\mu})} \\
=& D_{i, \min}(k_{\lambda}, k_{\mu}) + \underbrace{C^{a}_{i, \min}(k_{\lambda}, k_{\mu}) + C^{\beta}_{i, \min}(k_{\lambda}, k_{\mu})}_{\Phi_{i, \min}(k_{\lambda}, k_{\mu})} \geq 0. \label{eq:dpcbf_proof_conv}
\end{align}
\end{subequations}
We will show that \eqref{eq:dpcbf_proof_conv} holds for both $i \in \{ 1, 2\}$, which together implies that \eqref{eq:nagumo_eq} is satisfied.

\textbf{Case 1 ($i=1$).}
We aim to verify \eqref{eq:dpcbf_proof_conv} for the subset
$\partial\mathcal C_{1}$ of the safety boundary. Dividing \eqref{eq:dpcbf_proof_conv} by $v$ yields:
\begin{align}
    &\inf_{\vx\in\partial\mathcal C_1} \frac{\lieder_f h(\vx)}{v} + \inf_{\vx\in\partial\mathcal C_1} \biggl|\frac{C^{a}(\vx)}{v}\biggr|a_{\max} \nonumber \\ 
  & \qquad \qquad \qquad \qquad \quad +
  \inf_{\vx\in\partial\mathcal C_1} \biggl| \frac{C^{\beta}(\vx)}{v} \biggr|\beta_{\max}
  \geq 0 \label{eq:dpcbf_proof_case1}
\end{align}
By $| \sin\tilde{\theta} | \ge \bar{s}$ in \eqref{eq:dpcbf_split}, we have $\inf_{\vx \in \partial \calC_{1}} \Bigl| \frac{C^{a}(\vx)}{v} \Bigr| = 0$ (see Appendix~\autoref{app:case1}). Also, we show that the control term for $\beta$ and the drift term is both lower-bounded by $\Phi_{1,\min} (k_\mu)$ and $D_{1,\min}( k_\lambda, k_\mu)$, respectively, and they depend on the hyperparameters~$k_\lambda$ and $k_\mu$:
\begin{align}
    \inf_{\vx \in \partial \calC_{1}}\Bigl|\frac{C^{\beta}(\vx)}{v} \, \Bigr|\beta_{\max} & \coloneqq \Phi_{1,\min} (k_\mu) > 0 \\
 \inf_{\vx\in\partial\mathcal C_1}\frac{\lieder_f h(\vx)}{v} & \coloneqq D_{1,\min}( k_\lambda, k_\mu).
\end{align}
Therefore, the following condition is sufficient to satisfy \eqref{eq:dpcbf_proof_case1}:
\begin{equation} \label{eq:case1_result}
    \Phi_{1,\min}(k_\mu)\geq -D_{1,\min}(k_\lambda, k_\mu).
\end{equation}

\textbf{Case 2 ($i=2$).}
Similarly, we show \eqref{eq:dpcbf_proof_conv} for $\partial\mathcal C_{2}$.
\begin{align} \label{eq:dpcbf_proof_case2}
    &\inf_{\vx\in\partial\mathcal C_2} L_f h(\vx)
  +
  \inf_{\vx\in\partial\mathcal C_2} |C^{a}(\vx)|\,a_{\max} \nonumber \\
  &\qquad \qquad \qquad \qquad \quad +
  \inf_{\vx\in\partial\mathcal C_2} |C^{\beta}(\vx)|\,\beta_{\max} \ge 0
\end{align}
By $| \sin\tilde{\theta} | < \bar{s}$ in \eqref{eq:dpcbf_split}, we have $\inf_{\vx \in \partial \calC_{2}} |C^{\beta}(\vx)| = 0$ (see Appendix~\autoref{app:case2}). Also, each of the remaining terms are lower-bounded:
\begin{align}
    \inf_{\vx \in \partial \calC_{2}} | C^{a}(\vx)| a_{\max} & \coloneqq \Phi_{2,\min} (k_\lambda) > 0 \\
 \inf_{\vx\in\partial\calC_2} \lieder_f h(\vx) & \coloneqq D_{2,\min}( k_\lambda, k_\mu).
\end{align}
Therefore, the following condition is sufficient to satisfy \eqref{eq:dpcbf_proof_case2}:
\begin{equation} \label{eq:case2_result}
    \Phi_{2,\min}(k_\lambda)\geq -D_{2,\min}(k_\lambda, k_\mu).
\end{equation}

\begin{theorem}\label{thm:dpcbf_validity}
Under Assumptions~\ref{assum:robot_speed_bound}-\ref{assum:clearance}, the DPCBF is valid for System~\eqref{eq:bicycle_simplified_sys} under the input constraints, if there exist parameters $k_{\lambda}$ and $k_{\mu}$ that satisfy \eqref{eq:case1_result} and \eqref{eq:case2_result}.
\end{theorem}
\begin{proof}
A full proof with step-by-step derivation of each term can be found in Appendix~\autoref{app:case1}-\autoref{app:case2}, and how to find a feasible set of parameters $(k_\lambda, k_\mu)$ that satisfies both \eqref{eq:case1_result} and \eqref{eq:case2_result} are shown in Appendix~\autoref{app:result}.
\end{proof}

\begin{figure*}[t]
    \centering    \includegraphics[width=0.99\linewidth]{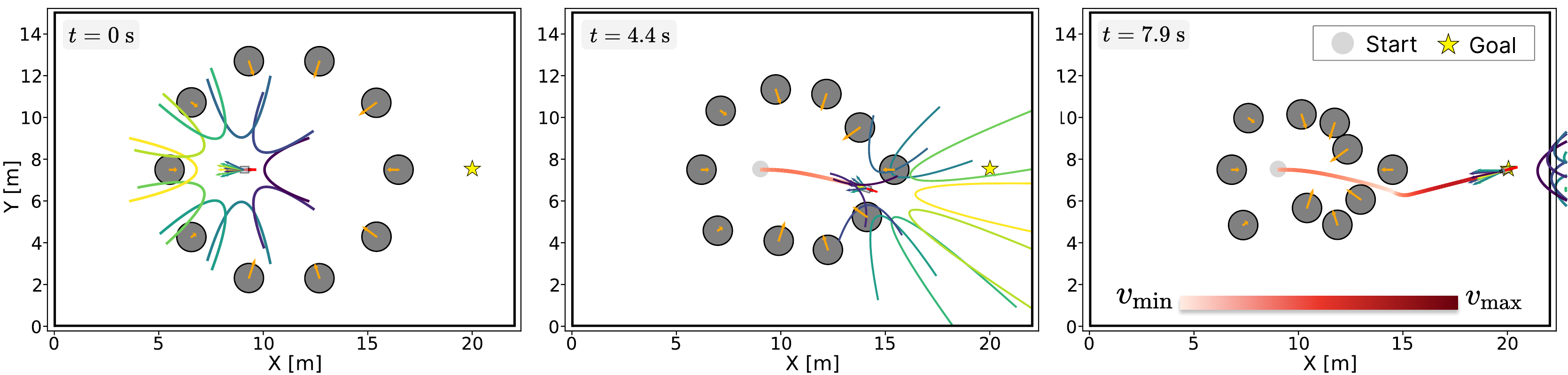}
    \caption{Demonstration of the proposed DPCBF’s navigation behavior in a surrounded environment with ten dynamic obstacles. For the same configuration, collision cone-based methods are infeasible, as shown in Fig~\ref{fig:c3bf_cases}d. (Left) At $t = 0$ s, with $v(t_{0}) = 0.5$ m/s, the QP with DPCBF constraints finds a feasible solution even though the robot is surrounded by unsafe sets, by keeping the relative-velocity vectors outside the dynamic parabolic boundaries, allowing it to proceed safely. (Center) By $t = 4.4$ s, the robot successfully maneuvers through a narrow passage. This is possible due to the less conservative formulation of DPCBF, which provides the necessary control flexibility in confined spaces. (Right) The robot safely navigates through the obstacles and reaches the goal at $t = 7.9$ s.}
    \label{fig:dpcbf_surround}
\vspace{-5pt}
\end{figure*}

\begin{remark}
As is common in CBF analysis, including the prior works we evaluate against, the safety guarantee in Theorem~\ref{thm:dpcbf_validity} holds for a single CBF constraint, corresponding to one obstacle. For methods on composing multiple CBF constraints into a single constraint, we refer the readers to \cite{breeden_compositions_2023,garg_advances_2024}. A formal investigation into the composition of multiple DPCBFs under input constraints is outside the scope of this paper. We evaluate the performance of our DPCBF-based controller against the compared methods using a QP with multiple constraints in \autoref{sec:results}.
\end{remark}

\section{RESULTS \label{sec:results}}
\subsection{Experimental setup}

We conduct a series of simulation experiments to evaluate the performance of our proposed DPCBF and compare it against state-of-the-art baseline methods. The primary goal is to assess the ability of DPCBF to maintain safety while reducing conservatism, particularly in challenging scenarios with multiple dynamic obstacles. All experiments are performed in a simulated 2D environment. The robot is modeled as a kinematic bicycle~\eqref{eq:bicycle_simplified_sys} with parameters specified in \autoref{tab:sim_parameters}. Dynamic obstacles are modeled as discs with varying radii and move with constant velocity. The nominal controller $\vu_{\textup{ref}}$ is a simple proportional controller that drives the robot towards a goal location. We compare our DPCBF against three established CBF methods for dynamic obstacle avoidance:

\textbf{(i) C3BF~\cite{tayal_control_2024}:} The collision-cone based CBF described in \autoref{sec:c3bf}. 

\textbf{(ii) MA-CBF-VO~\cite{roncero_multiagent_2025}:} This method uses a velocity obstacle formulation for guidance and a separate, distance-based CBF to formally guarantee safety. To avoid the conservative behavior of VO approaches, the VO constraint is relaxed into a soft constraint by using a slack variable in the optimization's objective function, while the distance-based CBF remains a hard constraint for collision avoidance.

\textbf{(iii) Dynamic zone-based CBF~\cite{wang_safe_2025}:} This approach modulates a circular safety zone around each obstacle based on the relative motion between the robot and the obstacle. The radius of this zone dynamically expands only when the robot and an obstacle are moving toward each other.

For all methods, the safety constraints are enforced via the CBF-QP formulation. We evaluate performance based on four key metrics: (i) Success rate: the percentage of trials where the robot reaches the goal without collision or infeasibility. (ii) Infeasible rate: the percentage of trials where the CBF-QP becomes infeasible, leading to mission failure. (iii) Collision rate: the percentage of trials where the robot's body intersects with an obstacle. (iv) QP cost: the total amount of deviation from the reference control input, calculated as the cumulative sum of the instantaneous QP cost, $\|\vu - \vu_{\textup{ref}}\|_2^{2}$, over the trajectory, where a lower QP cost implies a more efficient and less conservative method.
\begin{table}[t]
\centering
\begin{tabular}{l|cc|cc}
\toprule
Parameter & \multicolumn{2}{c|}{Bicycle Robot} & \multicolumn{2}{c}{Obstacles}  \\ 
\cmidrule(l){1-5}
Maximum velocity & \multicolumn{2}{c|}{3.5 [m/s]} & \multicolumn{2}{c}{1.2 [m/s]}   \\
Minimum velocity & \multicolumn{2}{c|}{0.2 [m/s]} & \multicolumn{2}{c}{0 [m/s]} \\
Maximum sensing range & \multicolumn{2}{c|}{15 [m]} & \multicolumn{2}{c}{-}   \\
$a_{\textup{max}}$ & \multicolumn{2}{c|}{5.0 [m/s$^{2}$]} & \multicolumn{2}{c}{-}  \\ 
$\beta_{\textup
max}$ & \multicolumn{2}{c|}{0.28 [rad]} & \multicolumn{2}{c}{-} \\ 
Max/Min radius & \multicolumn{2}{c|}{0.3 / - [m]} & \multicolumn{2}{c}{0.7 / 0.1 [m]} \\ 
Rear axes distance $\ell_r$  & \multicolumn{2}{c|}{0.2 [m]} & \multicolumn{2}{c}{-}  \\ 
Safety buffer $s$ & \multicolumn{2}{c|}{1.05} & \multicolumn{2}{c}{-} \\
\bottomrule
\end{tabular}
\caption{Main parameters for the simulation studies.}
\vspace{-6pt}
 \label{tab:sim_parameters}
\end{table}

\begin{figure*}[t]
    \centering    \includegraphics[width=0.99\linewidth]{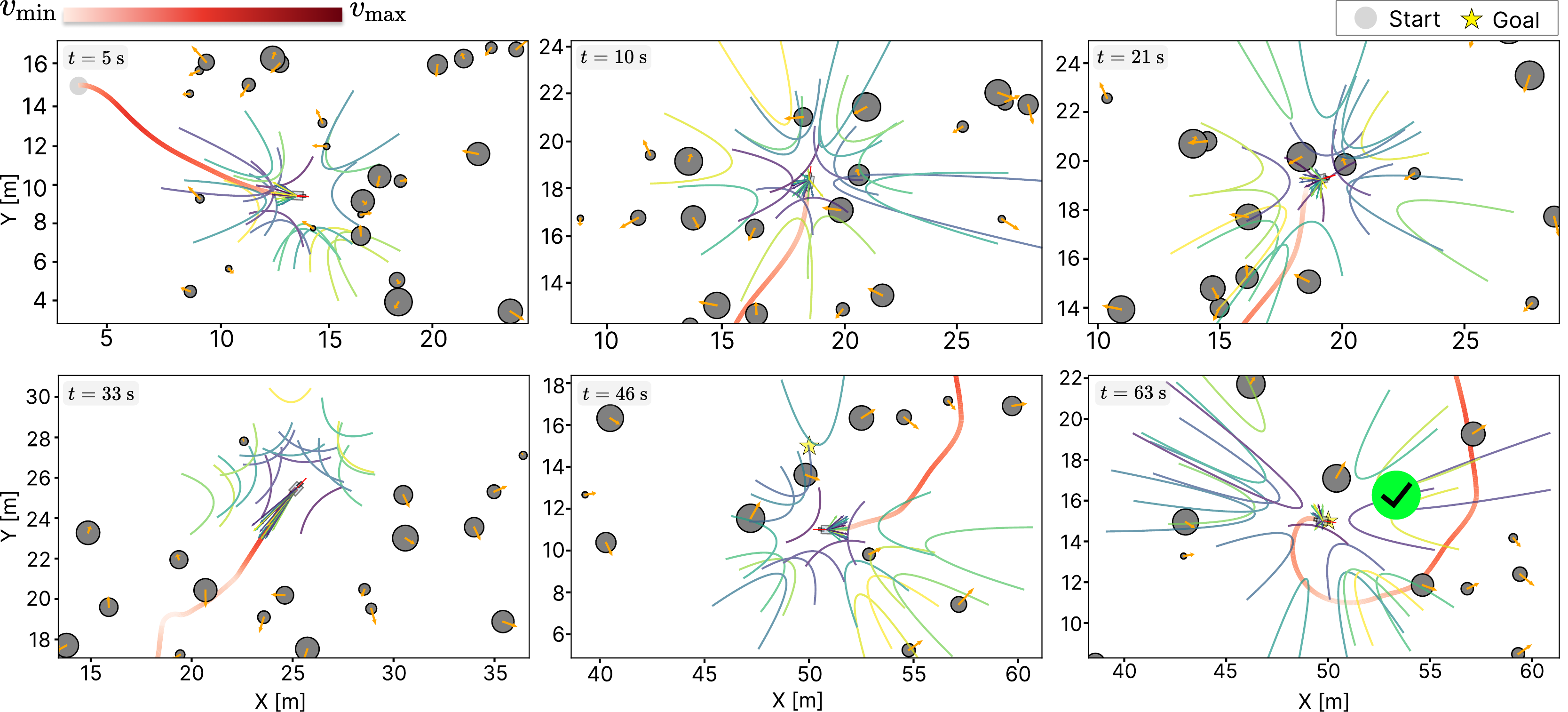}
    \caption{A visualization of a successful navigation scenario using DPCBF with 100 dynamic obstacles, drawn from the statistical results in Fig.~\ref{fig:performance_results}. All other compared baseline methods failed in this same configuration.}
    \label{fig:dpcbf_trialexp_snapshot}
\vspace{-5pt}
\end{figure*}
\subsection{Comparison with C3BF}


We first demonstrate a crucial qualitative comparison in Fig.~\ref{fig:dpcbf_surround}, directly addressing the failure case for C3BF shown in Fig.~\ref{fig:c3bf_cases}d. In this challenging scenario, the robot is initially surrounded by obstacles. While C3BF becomes infeasible due to the complete overlap of collision cones, DPCBF successfully finds a path to the goal. Although the robot is similarly enveloped by parabolic safety boundaries, the dynamic nature of DPCBF provides a key advantage. Specifically, the state-dependent term $\mu(\vx)$ in \eqref{eq:dpcbf_func_def} creates sufficient feasible space for the relative velocity in the CBF-QP. This directly illustrates how DPCBF overcomes the conservatism of cone-based methods.

\begin{figure}[t]
    \centering    \includegraphics[width=0.99\linewidth]{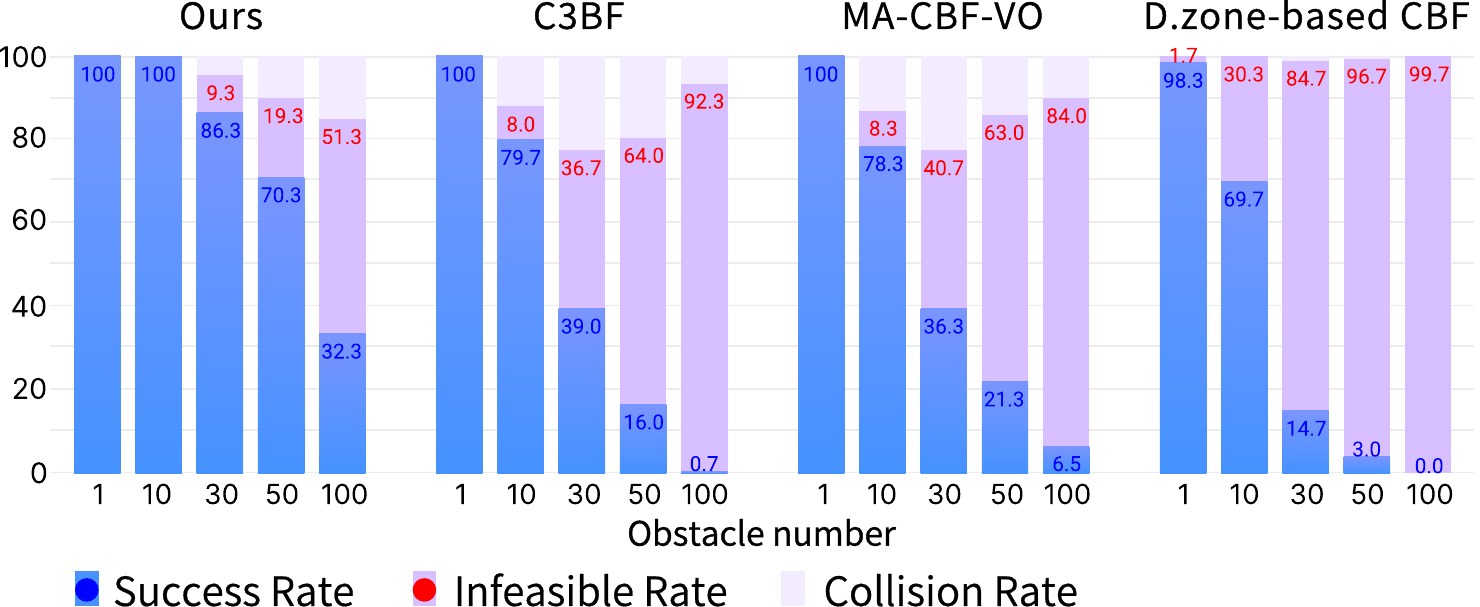}
    \caption{Performance comparison of success, infeasible, and collision rates for our method and three baselines as the number of obstacles increases from 1 to 100. Each bar represents the average of 300 trials, conducted across three scenarios with varying maximum obstacle radii (0.3, 0.5 and 0.7 m). The results highlight that our approach outperforms other state-of-the-art CBF methods by maintaining a high success rate in dense environments where baselines frequently become infeasible.}
    \label{fig:performance_results}
\vspace{-5pt}
\end{figure}

 \begin{figure}[t]
    \centering    \includegraphics[width=0.99\linewidth]{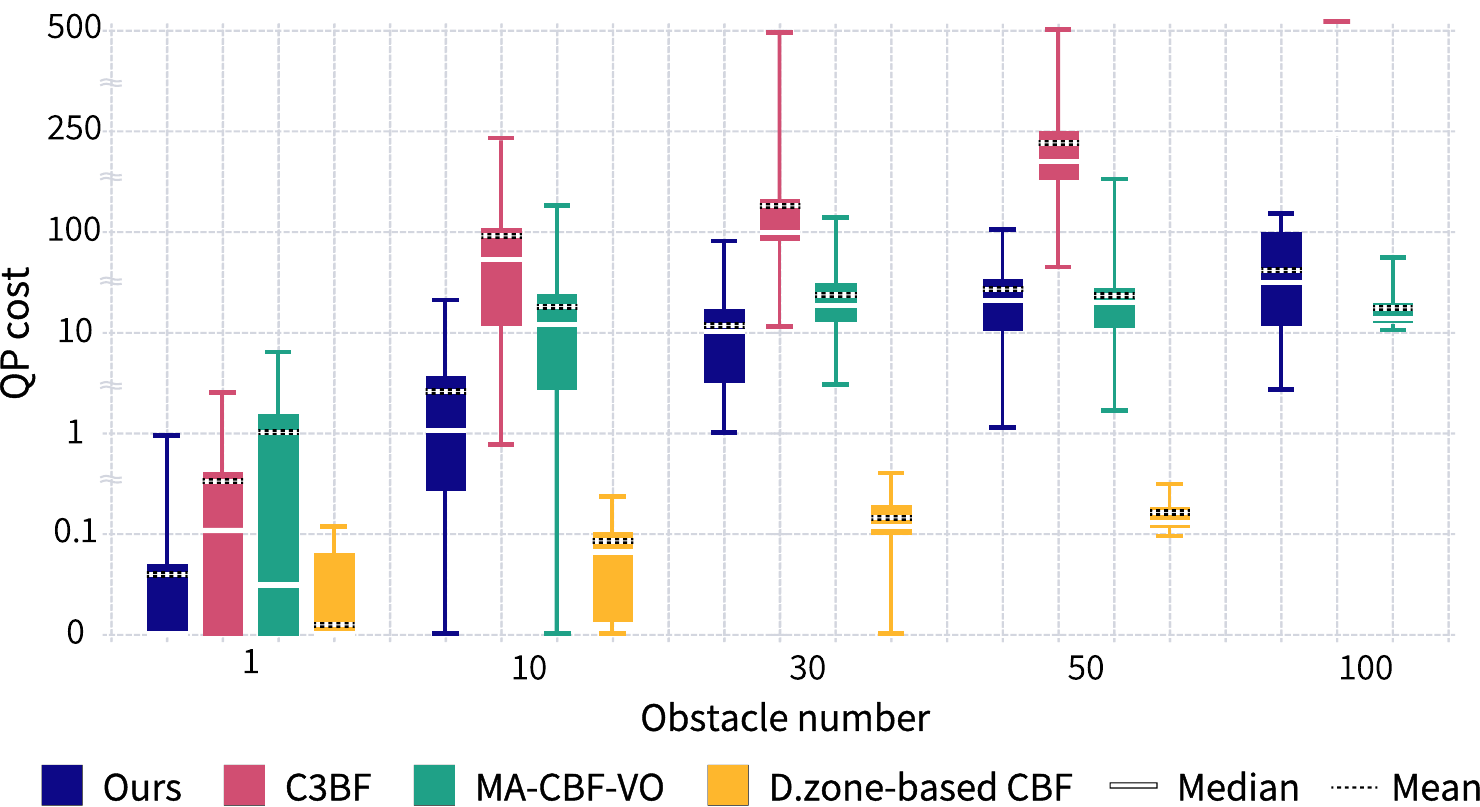}
     \caption{Control intervention, measured by QP cost $\|\vu-\vu_{\textup{ref}}\|^{2}_{2}$, is plotted against obstacle density for each method. Lower costs indicate greater efficiency and less conservative behavior.}
    \label{fig:box_plot}
\vspace{-5pt}
\end{figure}
\subsection{Experimental Results}

\textbf{Performance Analysis in Dense Dynamic Environments.} To test the core hypothesis that DPCBF alleviates infeasibility issues while ensuring safety, we simulate navigation in environments with an increasing number of dynamic obstacles, from 1 to 100. The results are summarized in Fig.~\ref{fig:performance_results}.
We first evaluate the methods in single-obstacle scenarios, where the formal safety guarantee holds for all methods except for Dynamic zone-based CBF. As expected from the theoretical guarantee, all such methods achieve a 100\% success rate. The Dynamic zone-based CBF exhibits a 1.7\% infeasibility rate because it is not a valid CBF for the kinematic bicycle model, regardless of the number of constraints. The performance of the compared methods drops dramatically as the number of dynamic obstacles increases, resulting in frequent QP infeasibility or even collisions. Notably, DPCBF achieves a 100\% success rate even in the 10-obstacle cases. This shows that the collision-cone based methods~\cite{tayal_control_2024,roncero_multiagent_2025} suffer in obstacle-dense environments, where overlapping collision cones severely constrain the feasible control space, leading to frequent QP failures.

\textbf{Analysis on Conservatism.} Fig.~\ref{fig:box_plot} details the QP cost for each method. Our DPCBF consistently exhibits the lower median and mean QP cost, navigating complex scenarios with minimal deviation from the reference controller. In contrast, C3BF requires the largest control interventions.
This reveals a fundamental design limitation that becomes prominent in scenarios with multiple dynamic obstacles: overlapping collision cones overly shrink the safe set. Consequently, the QP with C3BF constraints is forced to either decelerate constantly to maintain the minimum velocity or take a large detour from the optimal trajectory, leading to a longer time to reach the goal. Although the Dynamic zone-based CBF appears to have the lowest QP cost, it is highly prone to infeasibility, as shown in Fig.~\ref{fig:performance_results}. In scenarios with over 50 obstacles, its success rate drops to nearly 0\%. Furthermore, while MA-CBF-VO shows a QP cost comparable to the proposed DPCBF, it has a higher infeasibility and collision rate. This is because its VO constraints are soft constraints that are often relaxed in obstacle-dense environments.

\begin{figure}[t]
    \centering    \includegraphics[width=0.99\linewidth]{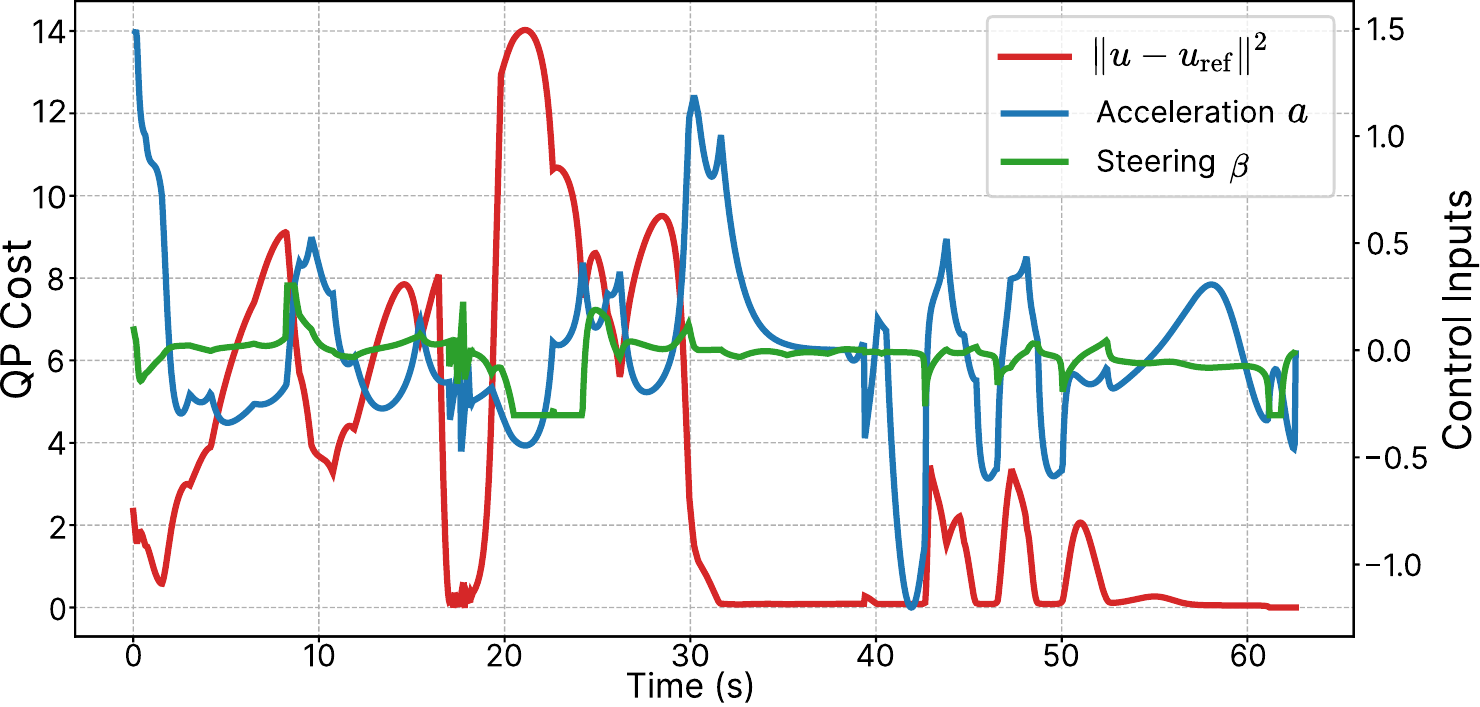}
    \caption{ QP cost and resulting control inputs over time for the dense scenario from Fig.~\ref{fig:dpcbf_trialexp_snapshot}. The peaks in QP cost near $t = 21$~s correspond to the most obstacle-dense moments of the trajectory shown in Fig.~\ref{fig:dpcbf_trialexp_snapshot}.}
    \label{fig:dpcbf_qp_cost}
\vspace{-5pt}
\end{figure}
\textbf{Qualitative Trajectory Analysis.} Fig.~\ref{fig:dpcbf_trialexp_snapshot} visualizes a challenging navigation scenario in which the DPCBF-based QP guides the robot through a dense group of 100 dynamic obstacles with a maximum obstacle radius of $r_{\textup{obs},\max} = 0.7$~m. We also visualize the velocity of the kinematic bicycle along its trajectory, with the corresponding QP cost and control inputs shown in Fig.~\ref{fig:dpcbf_qp_cost}. This demonstrates that DPCBF constraints guide the CBF-QP to effectively adjust both longitudinal and lateral motion around multiple obstacles, successfully performing safe navigation. Importantly, at snapshots taken at $t = 10$~s and $t = 21$~s, the union of the unsafe sets does not render the feasible space empty, whereas methods based on VO or collision cone would be infeasible in the same configurations. The robot is also able to regain high velocity at $t = 33$~s when the obstacles are no longer driving towards it. These examples highlight how DPCBF actively modifies the nominal control inputs to guarantee safety without being overly conservative.

\section{CONCLUSION \label{sec:conclusion}}

 In this paper, we introduced the Dynamic Parabolic Control Barrier Function (DPCBF), a novel CBF formulation for nonholonomic robots navigating in dynamic environments. By defining a safety boundary with a parabola that can adapt based on both relative distance and velocity, DPCBF generates a less conservative constraint that significantly improves the feasibility of the corresponding QP. Extensive simulations validated our approach, demonstrating higher navigation success rates in dense environments compared to state-of-the-art methods, particularly in challenging scenarios with up to 100 obstacles where cone-based approaches fail. Future work will focus on implementing DPCBF on physical hardware and investigating its extension to other complex dynamical systems.

\addtolength{\textheight}{0 cm}   





\bibliographystyle{IEEEtran}
\typeout{}
\bibliography{references.bib}

@inproceedings{wang_safe_2025,
	title = {Safe {Navigation} in {Uncertain} {Crowded} {Environments} {Using} {Risk} {Adaptive} {CVaR} {Barrier} {Functions}},
	doi = {10.48550/arXiv.2504.06513},
	urldate = {2025-08-22},
	booktitle = {{IEEE}/{RSJ} {International} {Conference} on {Intelligent} {Robots} and {Systems} ({IROS})},
	author = {Wang, Xinyi and Kim, Taekyung and Hoxha, Bardh and Fainekos, Georgios and Panagou, Dimitra},
	year = {2025},
	pages = {7669--7676},
}

@article{chakravarthy_obstacle_1998,
	title = {Obstacle avoidance in a dynamic environment: a collision cone approach},
	volume = {28},
	issn = {1558-2426},
	shorttitle = {Collision-{Cone}},
	doi = {10.1109/3468.709600},
	number = {5},
	urldate = {2025-09-13},
	journal = {IEEE Transactions on Systems, Man, and Cybernetics - Part A: Systems and Humans},
	author = {Chakravarthy, A. and Ghose, D.},
	year = {1998},
	pages = {562--574},
}

@inproceedings{breeden_compositions_2023,
	title = {Compositions of {Multiple} {Control} {Barrier} {Functions} {Under} {Input} {Constraints}},
	doi = {10.23919/ACC55779.2023.10156625},
	urldate = {2025-09-13},
	booktitle = {American {Control} {Conference} ({ACC})},
	author = {Breeden, Joseph and Panagou, Dimitra},
	year = {2023},
	pages = {3688--3695},
}

@article{huang_dynamic_2025,
	title = {Dynamic {Collision} {Avoidance} {Using} {Velocity} {Obstacle}-{Based} {Control} {Barrier} {Functions}},
	volume = {33},
	issn = {1558-0865},
	doi = {10.1109/TCST.2025.3546076},
	number = {5},
	urldate = {2025-09-04},
	journal = {IEEE Transactions on Control Systems Technology},
	author = {Huang, Jihao and Zeng, Jun and Chi, Xuemin and Sreenath, Koushil and Liu, Zhitao and Su, Hongye},
	year = {2025},
	pages = {1601--1615},
}

@article{sadraddini_provably_2017,
	title = {Provably {Safe} {Cruise} {Control} of {Vehicular} {Platoons}},
	volume = {1},
	issn = {2475-1456},
	doi = {10.1109/LCSYS.2017.2713772},
	number = {2},
	urldate = {2025-08-17},
	journal = {IEEE Control Systems Letters},
	author = {Sadraddini, Sadra and Sivaranjani, S. and Gupta, Vijay and Belta, Calin},
	year = {2017},
	pages = {262--267},
}

@inproceedings{samavati_optimal_2017,
	title = {An optimal motion planning and obstacle avoidance algorithm based on the finite time velocity obstacle approach},
	doi = {10.1109/AISP.2017.8324091},
	urldate = {2025-08-17},
	booktitle = {Artificial {Intelligence} and {Signal} {Processing} {Conference} ({AISP})},
	author = {Samavati, Sepehr and Zarei, Mojaba and Masouleh, Mehdi Tale},
	year = {2017},
	pages = {250--255},
}

@inproceedings{thontepu_collision_2023,
	title = {Collision {Cone} {Control} {Barrier} {Functions} for {Kinematic} {Obstacle} {Avoidance} in {UGVs}},
	doi = {10.1109/ICC61519.2023.10442173},
	urldate = {2025-08-17},
	booktitle = {Indian {Control} {Conference} ({ICC})},
	author = {Thontepu, Phani and Goswami, Bhavya Giri and Tayal, Manan and Singh, Neelaksh and P I, Shyam Sundar and M G, Shyam Sundar and Sundaram, Suresh and Katewa, Vaibhav and Kolathaya, Shishir},
	year = {2023},
	pages = {293--298},
}

@article{fiorini_motion_1998,
	title = {Motion {Planning} in {Dynamic} {Environments} {Using} {Velocity} {Obstacles}},
	volume = {17},
	issn = {0278-3649},
	doi = {10.1177/027836499801700706},
	language = {EN},
	number = {7},
	urldate = {2025-08-17},
	journal = {The International Journal of Robotics Research},
	author = {Fiorini, Paolo and Shiller, Zvi},
	year = {1998},
	pages = {760--772},
}

@article{kim_study_2016,
	title = {Study on optimal velocity selection using velocity obstacle ({OVVO}) in dynamic and crowded environment},
	volume = {40},
	issn = {1573-7527},
	doi = {10.1007/s10514-015-9520-6},
	language = {en},
	number = {8},
	urldate = {2025-08-17},
	journal = {Autonomous Robots},
	author = {Kim, Mingeuk and Oh, Jun-Ho},
	year = {2016},
	pages = {1459--1470},
}

@inproceedings{tayal_control_2024,
	title = {Control {Barrier} {Functions} in {Dynamic} {UAVs} for {Kinematic} {Obstacle} {Avoidance}: {A} {Collision} {Cone} {Approach}},
	shorttitle = {Control {Barrier} {Functions} in {Dynamic} {UAVs} for {Kinematic} {Obstacle} {Avoidance}},
	doi = {10.23919/ACC60939.2024.10644548},
	urldate = {2025-08-17},
	booktitle = {American {Control} {Conference} ({ACC})},
	author = {Tayal, Manan and Singh, Rajpal and Keshavan, Jishnu and Kolathaya, Shishir},
	year = {2024},
	pages = {3722--3727},
}

@inproceedings{guy_clearpath_2009,
	title = {{ClearPath}: highly parallel collision avoidance for multi-agent simulation},
	isbn = {978-1-60558-610-6},
	shorttitle = {{FVO}},
	doi = {10.1145/1599470.1599494},
	urldate = {2025-05-12},
	booktitle = {{ACM} {SIGGRAPH}/{Eurographics} {Symposium} on {Computer} {Animation}},
	author = {Guy, Stephen. J. and Chhugani, Jatin and Kim, Changkyu and Satish, Nadathur and Lin, Ming and Manocha, Dinesh and Dubey, Pradeep},
	year = {2009},
	pages = {177--187},
}

@article{kim_visibilityaware_2025,
	title = {Visibility-{Aware} {RRT}* for {Safety}-{Critical} {Navigation} of {Perception}-{Limited} {Robots} in {Unknown} {Environments}},
	volume = {10},
	doi = {10.48550/arXiv.2406.07728},
	number = {5},
	urldate = {2025-02-21},
	journal = {IEEE Robotics and Automation Letters},
	author = {Kim, Taekyung and Panagou, Dimitra},
	year = {2025},
	pages = {4508--4515},
}

@inproceedings{haraldsen_safetycritical_2024,
	title = {Safety-{Critical} {Control} of {Nonholonomic} {Vehicles} in {Dynamic} {Environments} {Using} {Velocity} {Obstacles}},
	doi = {10.23919/ACC60939.2024.10644678},
	urldate = {2024-12-08},
	booktitle = {American {Control} {Conference} ({ACC})},
	author = {Haraldsen, Aurora and Wiig, Martin S. and Ames, Aaron D. and Pattersen, Kristin Y.},
	year = {2024},
	pages = {3152--3159},
}

@inproceedings{he_rulebased_2021,
	title = {Rule-{Based} {Safety}-{Critical} {Control} {Design} using {Control} {Barrier} {Functions} with {Application} to {Autonomous} {Lane} {Change}},
	doi = {10.23919/ACC50511.2021.9482848},
	urldate = {2024-11-13},
	booktitle = {American {Control} {Conference} ({ACC})},
	author = {He, Suiyi and Zeng, Jun and Zhang, Bike and Sreenath, Koushil},
	year = {2021},
	pages = {178--185},
}

@inproceedings{polack_kinematic_2017,
	title = {The kinematic bicycle model: {A} consistent model for planning feasible trajectories for autonomous vehicles?},
	shorttitle = {The kinematic bicycle model},
	doi = {10.1109/IVS.2017.7995816},
	urldate = {2024-11-13},
	booktitle = {{IEEE} {Intelligent} {Vehicles} {Symposium} ({IV})},
	author = {Polack, Philip and Altché, Florent and d'Andréa-Novel, Brigitte and de La Fortelle, Arnaud},
	year = {2017},
	pages = {812--818},
}

@inproceedings{tayal_collision_2024,
	title = {A {Collision} {Cone} {Approach} for {Control} {Barrier} {Functions}},
	shorttitle = {{C3BF}},
	doi = {10.48550/arXiv.2403.07043},
	urldate = {2024-11-13},
	booktitle = {{arXiv}:2403.07043},
	author = {Tayal, Manan and Goswami, Bhavya Giri and Rajgopal, Karthik and Singh, Rajpal and Rao, Tejas and Keshavan, Jishnu and Jagtap, Pushpak and Kolathaya, Shishir},
	year = {2024},
}

@inproceedings{roncero_multiagent_2025,
	title = {Multi-{Agent} {Obstacle} {Avoidance} using {Velocity} {Obstacles} and {Control} {Barrier} {Functions}},
	urldate = {2024-11-07},
	booktitle = {{IEEE} {International} {Conference} on {Robotics} and {Automation} ({ICRA})},
	author = {Roncero, Alejandro Sánchez and Muchacho, Rafael I. Cabral and Ögren, Petter},
	year = {2025},
	pages = {6638--6644},
}

@article{garg_advances_2024,
	title = {Advances in the {Theory} of {Control} {Barrier} {Functions}: {Addressing} practical challenges in safe control synthesis for autonomous and robotic systems},
	volume = {57},
	issn = {1367-5788},
	doi = {10.1016/j.arcontrol.2024.100945},
	urldate = {2024-09-06},
	journal = {Annual Reviews in Control},
	author = {Garg, Kunal and Usevitch, James and Breeden, Joseph and Black, Mitchell and Agrawal, Devansh and Parwana, Hardik and Panagou, Dimitra},
	year = {2024},
	pages = {100945},
}

@article{ames_control_2017,
	title = {Control {Barrier} {Function} {Based} {Quadratic} {Programs} for {Safety} {Critical} {Systems}},
	volume = {62},
	issn = {1558-2523},
	doi = {10.1109/TAC.2016.2638961},
	number = {8},
	urldate = {2024-09-06},
	journal = {IEEE Transactions on Automatic Control},
	author = {Ames, Aaron D. and Xu, Xiangru and Grizzle, Jessy W. and Tabuada, Paulo},
	year = {2017},
	pages = {3861--3876},
}

@article{zhang_gcbf_2025,
	title = {{GCBF}+: {A} {Neural} {Graph} {Control} {Barrier} {Function} {Framework} for {Distributed} {Safe} {Multi}-{Agent} {Control}},
	volume = {41},
	issn = {1941-0468},
	shorttitle = {{GCBF}+},
	doi = {10.1109/TRO.2025.3530348},
	urldate = {2024-05-03},
	journal = {IEEE Transactions on Robotics},
	author = {Zhang, Songyuan and So, Oswin and Garg, Kunal and Fan, Chuchu},
	year = {2025},
	pages = {1533--1552},
}

@inproceedings{ahmadi_safe_2019,
	title = {Safe {Policy} {Synthesis} in {Multi}-{Agent} {POMDPs} via {Discrete}-{Time} {Barrier} {Functions}},
	doi = {10.1109/CDC40024.2019.9030241},
	urldate = {2024-04-14},
	booktitle = {{IEEE} {Conference} on {Decision} and {Control} ({CDC})},
	author = {Ahmadi, Mohamadreza and Singletary, Andrew and Burdick, Joel W. and Ames, Aaron D.},
	year = {2019},
	pages = {4797--4803},
}

@inproceedings{xiao_control_2019,
	title = {Control {Barrier} {Functions} for {Systems} with {High} {Relative} {Degree}},
	shorttitle = {hocbf},
	doi = {10.1109/CDC40024.2019.9029455},
	urldate = {2023-12-05},
	booktitle = {{IEEE} {Conference} on {Decision} and {Control} ({CDC})},
	author = {Xiao, Wei and Belta, Calin},
	year = {2019},
	pages = {474--479},
}

@inproceedings{zeng_safetycritical_2021,
	title = {Safety-{Critical} {Model} {Predictive} {Control} with {Discrete}-{Time} {Control} {Barrier} {Function}},
	shorttitle = {mpc-cbf},
	doi = {10.23919/ACC50511.2021.9483029},
	urldate = {2024-01-18},
	booktitle = {American {Control} {Conference} ({ACC})},
	author = {Zeng, Jun and Zhang, Bike and Sreenath, Koushil},
	year = {2021},
	pages = {3882--3889},
}

\clearpage
\onecolumn
\appendix

Recall the proposed DPCBF parameterized by $k_{\lambda}$ and $k_{\mu}$ is the following:
\begin{equation}
    h(\vx; k_{\lambda}, k_{\mu}) = \tilde v_{\mathrm{rel},x} 
       + \lambda(\vx; k_{\lambda})\,\tilde v_{\mathrm{rel},y}^2
       + \mu(\vx; k_{\mu}),
\end{equation}
i.e.

\begin{equation}\label{eq:dpcbf_def}
\boxed{
    h(\vx; k_{\lambda}, k_{\mu})
= \cos\alpha \, v_{\mathrm{rel},x} + \sin\alpha\,v_{\mathrm{rel},y}
  + k_{\lambda}\,\frac{d(\vx)}{\|\vv_{\mathrm{rel}}\|} \,
    \bigl(-\sin\alpha \, v_{\mathrm{rel},x} + \cos\alpha \, v_{\mathrm{rel},y}\bigr)^2
  + k_{\mu}\,d(\vx).
  }
\end{equation}
\subsection{Useful Maths}

This section collects standard analytic inequalities repeatedly invoked in Appendix~\autoref{app:case1} and \autoref{app:case2}.

\begin{proposition}[Infimum sub-additivity]\label{prop:inf_subadd}
Let $\ell_1,\ell_2 \to \RealSpace$ be functions bounded below on a non-empty set $Z$. Then,
\begin{equation}
\inf_{z\in Z}\bigl(\ell_{1}(z)+\ell_{2}(z)\bigr)
\;\ge\;
\inf_{z\in Z}\ell_{1}(z) + \inf_{z\in Z} \ell_{2}(z).
\end{equation}
Equality holds when both infima are attained at a common point in $Z$;
otherwise the inequality is strict.
\end{proposition}

\begin{proposition}[Triangle Inequality]\label{prop:triangle}
For any $x,y\in\mathbb{R}$,
\begin{equation}
|x+y| \;\le\; |x| + |y|.
\end{equation}
This classical result is called the \emph{triangle inequality}; it states that, on the real line (and more generally in every normed space), the “length” of one side of a triangle does not exceed the sum of the lengths of the other two.
\end{proposition}

\begin{corollary}[Supremum form of Proposition~\ref{prop:triangle}]\label{cor:triangle_inf}
Let $Z\subset\mathbb{R}$ be a non-empty set. Then
\begin{equation}
\sup_{z\in Z} |\ell_{1}(z) + \ell_{2}(z)|
\;\le\;
\sup_{z\in Z} |\ell_{1}(z)|
\;+\;
\sup_{z\in Z} |\ell_{2}(z)|.
\end{equation}

\end{corollary}

\begin{proposition}[Reverse Triangle Inequality]\label{prop:reverse}
For any $x,y\in\mathbb{R}$,
\begin{equation}
|x+y| \;\ge\; |x| - |y|.
\end{equation}
This is the \emph{reverse triangle inequality}; it provides a lower bound on the absolute value of a sum in terms of the absolute values of its summands.
\end{proposition}

\begin{corollary}[Infimum form of Proposition~\ref{prop:reverse}]\label{cor:reverse_inf}
Let $Z\subset\mathbb{R}$ be a non-empty set, and let $\ell_{1}, \ell_{2}:Z\to\mathbb{R}$ be real-valued functions. Then,
\begin{equation}
\inf_{z\in Z} |\ell_{1}(z)+\ell_{2}(z)|
\;\ge\;
\inf_{z\in Z} |\ell_{1}(z)|
\;-\;
\sup_{z\in Z} |\ell_{2}(z)|.
\end{equation}
\end{corollary}

\begin{proposition}\label{prop:trigono_max}
For every real angle $\theta \in \mathbb{R}$,
\begin{equation}
    -\frac{1}{2} \le \sin\theta\cos\theta \le \frac{1}{2},
\qquad \text{equivalently} \qquad
\bigl| \sin \theta \cos\theta\bigr| \le \frac{1}{2}.
\end{equation}
Moreover, equality holds if and only if 
\begin{equation}
    \theta = \frac{\pi}{4} + k\frac{\pi}{2},
\qquad k \in \mathbb{Z}.
\end{equation}
\end{proposition}

\begin{proof}
The result follows from the identity $\sin 2\theta = 2\sin \theta \cos \theta$ and the bound $|\sin 2\theta| \leq 1$. Equality requires $|\sin 2\theta| = 1$, which occurs when $2 \theta = \frac{\pi}{2} + k \pi$ for any integer $k$.
\end{proof}

\subsection{Extra Notations}\label{app:extra_notation}

\begin{figure}[t]
    \centering
    \includegraphics[width=0.6\linewidth]{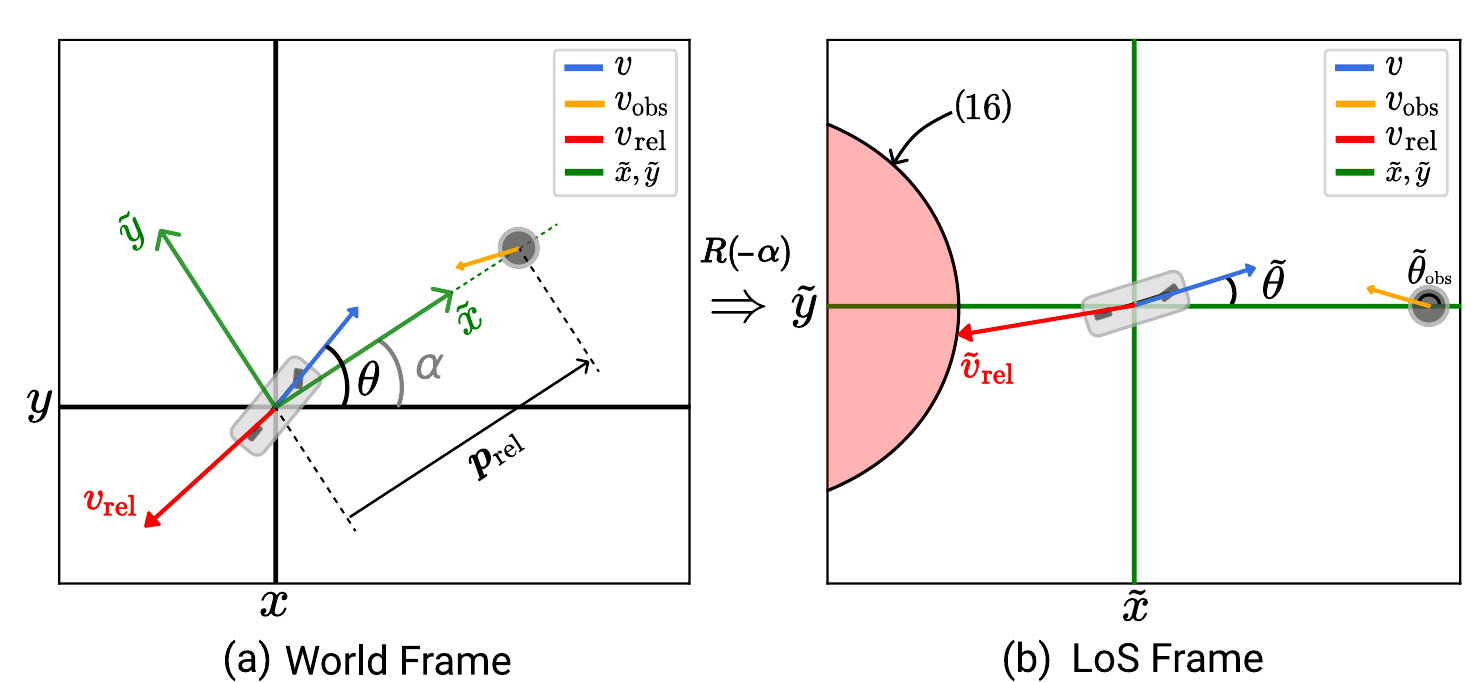}

    \caption{World frame (left) and LoS frame (right) geometries showing all relevant heading angles and relative velocity components ($\theta, \tilde \theta, \tilde \theta_{\textup{obs}} , \alpha , \psi , \tilde \psi$).}
    
    \label{fig:dpcbf_appendix_notation}
\end{figure}

To facilitate the subsequent analysis, this section introduces several key coordinate frames, angles, and a comprehensive summary of notation.

\paragraph{Coordinate frames}
As defined in \eqref{eq:dpcbf_rotation} in the main text, we rotate the world
frame by the Line-of-Sight (LoS) angle $\alpha$ so that the $\tilde x$–axis points from
the robot to the obstacle. 
All quantities expressed in this LoS frame are denoted with a tilde. In particular, the relative velocity $\vv_{\textup{rel}}$ transforms as $\tilde \vv_{\textup{rel}}= R(-\alpha) \vv_{\textup{rel}}$, with components:
\begin{equation}\label{eq:tilded_vel}
      \tilde v_{\textup{rel},x} = \| \vv_{\textup{rel}} \| \cos\tilde\psi,
  \quad
  \tilde v_{\textup{rel},y}= \| \vv_{\textup{rel}}\|\sin\tilde\psi,
  \qquad
  \tilde\psi = \psi-\alpha.
\end{equation}
Here, $\psi = \operatorname{atan2}(v_{\textup{rel},y}, v_{\textup{rel},x})$ is the heading of the relative velocity in the world frame.

\paragraph{Heading angles}
In the LoS frame, heading angles are measured relative to the $\tilde x$-axis. The robot and the obstacle headings are therefore given by (see Fig.~\ref{fig:dpcbf_appendix_notation}):
\begin{equation}\label{eq:los_angles}
    \tilde\theta = \theta-\alpha,
    \qquad
    \tilde\theta_{\textup{obs}} = \theta_{\textup{obs}}-\alpha .
\end{equation}

The notation used throughout the Appendix is organized in \autoref{table:notation}.

\begin{table}[!ht]
\centering
\caption{Nomenclature}
\renewcommand{\arraystretch}{1.05}
\begin{tabular}{@{} l p{0.64\linewidth} @{}}
\toprule
\textbf{Symbol} & \textbf{Definition} \\
\midrule
\multicolumn{2}{@{}l}{\textit{Geometry:}}\\
$\vp_{\textup{rel}}$ & Relative position vector\\
$\vv_{\textup{rel}}$ & Relative velocity vector\\
$\alpha$ & Line–of–Sight (LoS) angle \\
$\psi$ & Relative velocity angle in the world frame\\
$\tilde v_{\textup{rel},x/y}$ & Components of $\vv_{\textup{rel}}$ in the LoS frame \\
$\tilde\psi$ & Relative velocity angle in the LoS frame (=$\psi-\alpha$) \\
$d(\vx)$ & Clearance, $\|\vp_{\textup{rel}}\|^2 - r^{2}$ \\
\addlinespace[2pt]
\multicolumn{2}{@{}l}{\textit{DPCBF and Parameters:}}\\
$k_\lambda,k_\mu$ & Positive tunable parameters \\
$\lambda(\vx; k_\lambda)$ & Curvature parameter, $k_\lambda d(\vx)/\|\vv_{\textup{rel}}\|$ \\
$\mu(\vx; k_\mu)$ & Vertex-shift parameter, $k_\mu d(\vx)$ \\
$h(\vx; k_{\lambda}, k_{\mu})$ & DPCBF candidate function \\
\addlinespace[2pt]
\multicolumn{2}{@{}l}{\textit{Input–Related Terms:}}\\
$C^{a}(\vx), C^{\beta}(\vx)$ & Lie derivative coefficients for  acceleration and steering \\
$a_{\max},\beta_{\max}$ & Admissible input bounds \\
\addlinespace[2pt]
\multicolumn{2}{@{}l}{\textit{Bounds and Constants:}}\\
$p_{\min},p_{\max}$ & Minimum/maximum relative distance\\
$d_{\min},d_{\max}$ & Minimum/maximum clearance \\
$\|\vv_{\textup{rel}}\|_{\min/\max}$ & Minimum/maximum relative speed\\
$\ell_r$ & Rear axle to CoM distance (bicycle model) \\
$\tilde\psi_{\max}$ & Upper bound on $|\tilde \psi|$ on the safety boundary, derived in \autoref{lemma:trigono_bound}\\
\bottomrule
\end{tabular}
\label{table:notation}
\end{table}

\subsection{Problem Formulation}\label{app:problem_formulation}

Our objective is to prove that the candidate barrier function~$h$ is a valid CBF as defined in \autoref{def:cbf_basic}. By Nagumo's Theorem, it is sufficient to verify the CBF condition~\eqref{cond:cbf_basic} on the boundary of the safe set, $\partial \cal{C}$.

\begin{remark}[CBF Condition Under Input Constraints]\label{remark:boxNagumo}
Let the set of admissible inputs for the System~\eqref{eq:bicycle_simplified_sys} be
\begin{equation}
  \mathcal U= \{ [a, \beta]^\top \mid |a| \le a_{\max}, |\beta| \le \beta_{\max} \}.
\end{equation}
The control authority, representing the maximum effect of the input on $\dot{h}(\vx)$, is given by:
\begin{equation}\label{eq:appendix_input term_deriv}
  \Phi (\vx) \coloneqq \sup_{\vu \in \cal{U}} \lieder_g h(\vx)\vu = \sup_{\vu \in \cal{U}} \begin{bmatrix} C^{a}(\vx) \\C^{\beta}(\vx)\end{bmatrix}^\top \vu
          = |C^{a}(\vx) | a_{\max} + | C^{\beta}(\vx) | \beta_{\max}.
\end{equation}
where the coefficients $C^{a}(\vx)$ and $C^{\beta}(\vx)$ are the components of the Lie derivative $\lieder_g h(\vx)$ as
\begin{subequations}
\begin{align}
    \big|C^{a}(\vx)\big| =&  \bigg| \Bigl[
        -1
        + k_\lambda \frac{d(\vx)}{\|v_{\textup{rel}}\|^{3}}  v_{\textup{obs}} \cos\tilde\theta_{\textup{obs}}\tilde v_{\textup{rel},y}^{2}
      \Bigr] \cos\tilde \theta  
      +  \Bigl[k_\lambda \frac{d(\vx)}{\|v_{\textup{rel}}\|^{3}} v_{\textup{obs}} \sin \tilde \theta_{\textup{obs}} \tilde v_{\textup{rel},y}^{2}
          -2 k_\lambda \frac{d(\vx)}{\|v_{\textup{rel}}\|}
            \tilde v_{\textup{rel},y}
      \Bigr] \sin\tilde \theta \nonumber \\
      &+ \Bigl[
        - k_\lambda \frac{d(\vx)}{\|v_{\textup{rel}}\|^{3}}
          v \tilde v_{\textup{rel},y}^{2}
      \Bigr] \bigg|, \\
    \big|C^{\beta}(\vx)\big| =& \bigg| v\Bigl[
            -\frac{\tilde v_{\textup{rel},y}}{\|p_{\textup{rel}}\|}
            + 2k_\lambda \frac{d(\vx)}{\| v_{\textup{rel}} \|} 
              \frac{\tilde v_{\textup{rel},y}\tilde v_{\textup{rel},x}} {\|p_{\textup{rel}}\|}
        + \frac{v}{\ell_r} \Bigl(k_\lambda \frac{d(\vx)}{\|v_{\textup{rel}}\|^3} v_{\textup{obs}} \sin\tilde\theta_{\textup{obs}} \tilde v_{\textup{rel},y}^{2}
        - 2 k_\lambda\frac{d(\vx)}{\|v_{\textup{rel}}\|} \tilde{v}_{\textup{rel},y}
        \Bigr)
      \Big] \cos\tilde \theta \nonumber
  \\
    &  + v\Bigl[
         \Bigl(
            k_\lambda \frac{\|p_{\textup{rel}}\|}{d(\vx)}
              \frac{\tilde{v}_{\textup{rel},y}^{2}}{\|v_{\textup{rel}}\|}
          + k_\mu \frac{\|p_{\textup{rel}}\|}{d(\vx)}\Bigr)
        + \frac{v}{\ell_r}\Bigl(
            1
          - k_\lambda \frac{d(\vx)}{\|v_{\textup{rel}}\|^{3}}v_{\textup{obs}} \cos\tilde \theta_{\textup{obs}} \tilde{v}_{\textup{rel},y}^{2}\Bigr)
      \Bigr]
      \sin\tilde \theta \bigg|.
\end{align}
\end{subequations}
If the following holds:
\begin{equation}\label{eq:appendix_nagumo}
  \lieder_f h(\vx; k_{\lambda}, k_{\mu}) + \Phi(\vx; k_{\lambda}, k_{\mu}) \;\ge\; 0,\qquad \forall \vx\in\partial\mathcal C,
\end{equation}
where
\begin{align} \label{eq:appendix_drift_term_deriv}
    \lieder_{f} h(\vx) = v\biggl[\biggl(
  -k_\lambda
     \frac{\|\vp_{\textup{rel}}\|}{d(\vx)}
     \frac{\tilde v_{\textup{rel},y}^2}{\|\vv_{\textup{rel}}\|}
     -k_\mu \frac{\|\vp_{\textup{rel}}\|}{d(\vx)} \biggr) \cos\tilde \theta + \biggl(2k_\lambda
     \frac{\tilde v_{\textup{rel},y}}{\|\vv_{\textup{rel}}\|}
     \frac{d(\vx)}{\|\vp_{\textup{rel}}\|}\tilde v_{\textup{rel},x}
     -\frac{\tilde v_{\textup{rel},y}}{\|\vp_{\textup{rel}}\|}
  \biggr)\sin\tilde \theta \biggr],
\end{align}
then $h$ is a valid CBF for System~\eqref{eq:bicycle_simplified_sys} and the safe set $\mathcal C$ is rendered forward invariant.
\end{remark}

Note, the validity of the DPCBF depends on the parameters $k_{\lambda}$ and $k_{\mu}$. Therefore, we formulate the following problem:

\begin{problem}[DPCBF Validity]\label{prob:validity_of_dpcbf}
Find parameters $k_{\lambda} > 0$ and $k_{\mu} > 0$ such that the CBF condition~\eqref{eq:appendix_nagumo} holds for all states on the safety boundary $\vx \in \partial \mathcal{C}$.
\end{problem}

To solve \autoref{prob:validity_of_dpcbf}, we partition the boundary $\partial \cal{C}$ and analyze each subset separately.

\begin{remark}[Sufficient Condition via Infimum Sub-additivity] \label{remark:partition}
    Let the boundary be partitioned as $\partial \calC = \partial \calC_{1} \cup \partial \calC_{2}$. The CBF condition~\eqref{eq:appendix_nagumo} mush hold over each subset, i.e., for $i \in \{1, 2\}$:
    \begin{equation}
    \inf_{\vx \in \partial \cal{C}_{i}}\big[\lieder_f h(\vx)+|C^{a}(\vx)|a_{\max}+|C^{\beta}(\vx)|\beta_{\max}\big] \ge 0 ,\quad i=1,2.
    \end{equation}
    Using the property of infimum sub-additivity (\autoref{prop:inf_subadd}), we have:
    \begin{align}
        &\inf_{\vx\in\partial\mathcal C_i}\big[\lieder_f h(\vx)+|C^{a}(\vx)|a_{\max}+|C^{\beta}(\vx)|\beta_{\max}\big] \\
        \text{\autoref{prop:inf_subadd}} \Rightarrow \quad 
        \geq & 
        \underbrace{\inf_{\vx\in\partial\mathcal C_i}\big[\lieder_f h(\vx)\big]}_{D_{i, \min}(k_{\lambda}, k_{\mu})}+ \underbrace{\inf_{\vx\in\partial\mathcal C_i}\big[|C^{a}(\vx)|a_{\max} \big]}_{C^{a}_{\min}(k_{\lambda}, k_{\mu})}  + \underbrace{\inf_{\vx\in\partial\mathcal C_i}\big[|C^{\beta}(\vx)|\beta_{\max}\big]}_{C^{\beta}_{\min}(k_{\lambda}, k_{\mu})}  \ge 0 ,\quad i = 1,2.
    \end{align}
    Therefore, a sufficient condition for verifying the DPCBF is to show that for each partition $i \in \{1,2\}$ :
    \begin{equation}\label{eq:dpcbf_condition_converted}
    D_{i, \min}(k_{\lambda}, k_{\mu})+ \underbrace{C^{a}_{\min}(k_{\lambda}, k_{\mu}) + C^{\beta}_{\min}(k_{\lambda}, k_{\mu})}_{\Phi_{i, \min}(k_{\lambda}, k_{\mu})}   \ge 0,\quad i=1,2.
    \end{equation}
    Our proof proceeds by proving \eqref{eq:dpcbf_condition_converted} holds, which together imply \eqref{eq:appendix_nagumo}.
\end{remark}

First, we demonstrate that the worst-case analysis can be restricted to a smaller, critical set of robot and obstacle headings, simplifying the search for the infimum.

\begin{lemma}[Critical-Heading Set]\label{lemma:dpcbf-heading-angle-envelopes}
The worst-case analysis of the CBF condition~\eqref{eq:dpcbf_condition_converted} occurs within the set~
\begin{equation}
    \mathcal{A} =\{ \vx \in \partial \mathcal C \mid \tilde\theta\in[-\frac{\pi}{2},\frac{\pi}{2}],
      \tilde\theta_{\textup{obs}}\in[\frac{\pi}{2},\frac{3\pi}{2}]\}.
\end{equation}
\end{lemma}

\begin{proof}
The CBF condition is most difficult to satisfy when the drift $\lieder_f h(\vx)$ is most negative, requiring maximal control authority $\Phi (\vx)$ to counteract it. This corresponds to the most dangerous geometric configurations. The time-to-collision (TTC), defined as $\tau = \|\vp_{\textup{rel}}\|/|\tilde v_{\textup{rel},x}|$, provides a measure of this risk. Since  the range of $\|\vp_{\textup{rel}}\|$ is fixed by \autoref{assum:clearance}, minimizing the TTC is equivalent to maximizing the magnitude of relative velocity $|\tilde v_{\textup{rel}, x}|$. From the definition $\tilde v_{\textup{rel}, x} = -v \cos \tilde \theta + v_{\textup{obs}} \cos \tilde \theta_{\textup{obs}} < 0$, the term $|\tilde v_{\textup{rel}, x}|$ is maximized when the robot heads towards the obstacle $(\cos \tilde \theta > 0)$ and the obstacle heads towards the robot $(\cos \tilde \theta_{\textup{obs}} < 0)$. This geometric condition precisely defines the critical-heading set $\mathcal{A}$.
\end{proof}

By \autoref{remark:partition} and \autoref{lemma:dpcbf-heading-angle-envelopes}, we can reformulate the following problem instead of \autoref{prob:validity_of_dpcbf}:

\begin{figure}[t]
    \centering
    \includegraphics[width=0.6\linewidth]{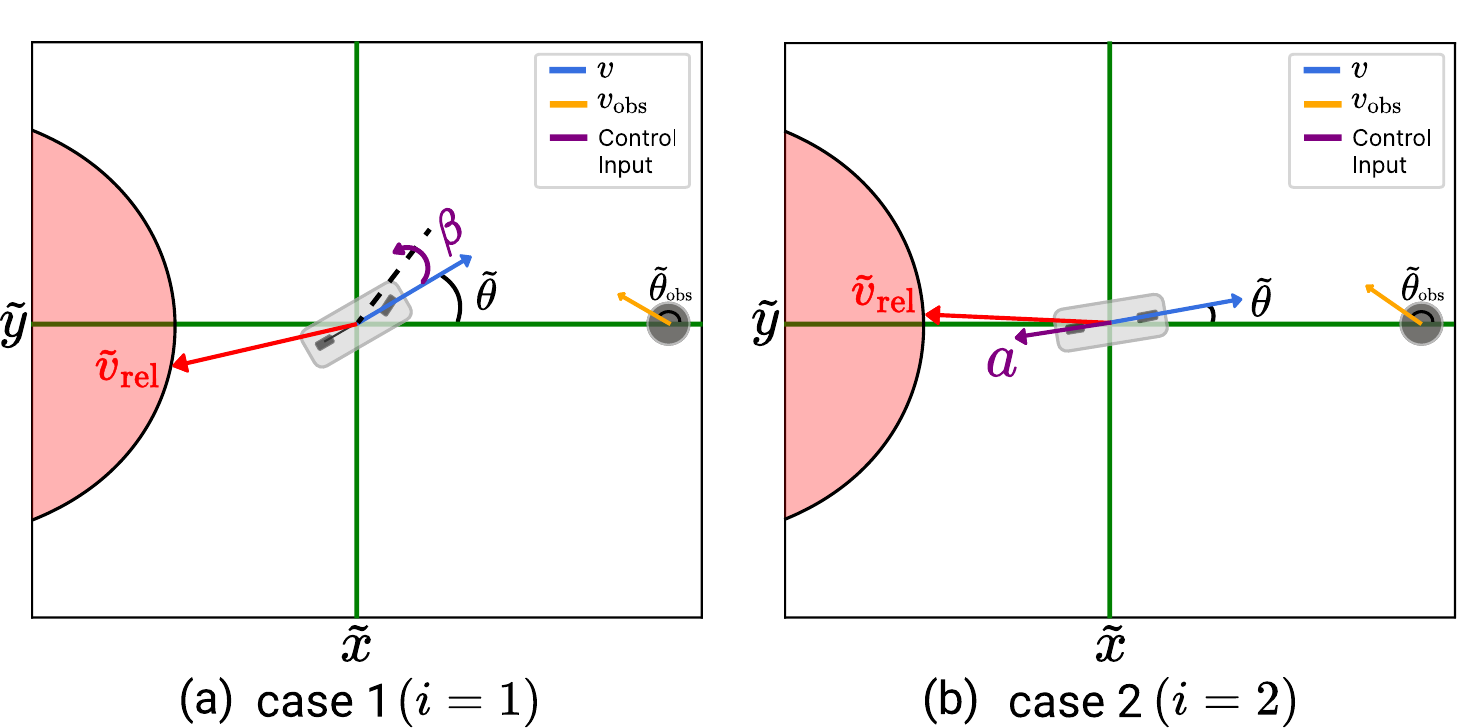}
    \caption{Illustration of the two boundary partitions used in the proof of \autoref{thm:dpcbf_validity}. (a) In the steering-dominant case ($i$=1), the steering input $\beta$ is the primary means of ensuring safety. (b) In the longitudinal-dominant case ($i$=2), the acceleration input $a$ is dominant.}
    \label{fig:dpcbf_appendix_proof_partition}
\end{figure}
\begin{problem}[Feasible Region Selection]\label{prob:gain_feasible}
Let us partition the boundary~$\partial\mathcal C$ according to the state-dependent threshold $\bar s \coloneqq \frac{v_{\textup{obs}}}{v}\sin \tilde \theta_{\textup{obs}} \in (-1,1)$. 
\begin{subequations}
\begin{align}
    \partial\mathcal C_{1} &= \{\vx \in \partial \mathcal{C}\mid |\sin\tilde\theta| \ge\bar s\}, \\
  \partial\mathcal C_{2} &=\{\vx \in \partial \mathcal{C} \mid |\sin\tilde\theta| < \bar{s} \}.
\end{align} 
\end{subequations}
By symmetry, it is sufficient to analyze the domain where $\sin \tilde \theta \geq 0$ and $\bar s \in [0, 1)$, which yields two cases:
\begin{subequations}
\begin{align}\label{eq:dpcbf_proof_partition}
    \partial\mathcal C_{1} &= \{\vx \in \partial \mathcal{C}\mid \sin\tilde\theta \ge\bar{s} \}, \\
  \partial\mathcal C_{2} &=\{\vx \in \partial \mathcal{C} \mid  \sin\tilde\theta <\bar s\}.
\end{align} 
\end{subequations}
Find positive parameters $k_\lambda$ and $k_\mu$ that satisfy the sufficient condition~\eqref{eq:dpcbf_condition_converted} on both $\partial \mathcal{C}_{1}$ and $\partial \mathcal{C}_{2}$.
\end{problem}

The partition is key to proof. The threshold $\bar s$ separates the boundary into two distinct regions, each corresponding to a different dominant control strategy (see Fig.~\ref{fig:dpcbf_appendix_proof_partition}):
\begin{itemize}
    \item \textbf{Steering-Dominant} Case $(\partial \mathcal{C}_{1})$: Here, the robot has a significant heading component towards the obstacle's path. Steering input~$(+\beta)$ is the most effective control action to generate lateral separation and ensure safety.
    \item \textbf{Longitudinal-Dominant} Case $(\partial \mathcal{C}_{2})$: Here, the robot's heading is nearly aligned with the line-of-sight vector. Deceleration input~$(-a)$ is the primary control action to manage the relative speed.
\end{itemize}
By evaluating these cases independently, we prove that the control authority is sufficient in each scenario. To proceed, we first derive several technical lemmas that establish bounds on key quantities on the safety boundary.

\begin{lemma}[Lower Bounded Relative Speed]
\label{lemma:boundary_coupling}
For any state on the boundary $\vx \in \partial \mathcal{C}$, the magnitude of relative velocity is coupled to the clearance from \autoref{assum:clearance} by
\begin{equation}\label{eq:vrel_exact}
  \|\vv_{\textup{rel}}\| =
  \frac{k_{\mu}d(\vx)}
        {-\cos\tilde\psi - k_{\lambda}d(\vx) \sin^{2}\tilde\psi}.
\end{equation}
Consequently, it is uniformly lower-bounded:
\begin{equation}\label{eq:vrel_lower}
      \|\vv_{\textup{rel}}\| \ge 
      \|\vv_{\textup{rel}}\|_{\min}
      \coloneqq k_{\mu} d_{\min} > 0.
\end{equation}
\end{lemma}

\begin{proof}
On the safety boundary $\partial \mathcal{C}$, we have $h(\vx) = 0$. Substituting the line-of-sight velocity components from \eqref{eq:tilded_vel} into the definition of $h(\vx)$ yields
\begin{align}
    &\tilde v_{\textup{rel},x}
  + k_\lambda \frac{d(\vx)}{\|\vv_{\textup{rel}}\|} \tilde v_{\textup{rel},y}^{2}
  + k_\mu d(\vx) = 0 \\
  \eqref{eq:tilded_vel} \Rightarrow \quad & \|\vv_{\textup{rel}}\| \cos \tilde \psi + k_{\lambda} d(\vx) \|\vv_{\textup{rel}}\| \sin^{2}\tilde \psi  + k_{\mu} d(\vx) =0.
  \label{eq:dpcbf_after_tilde_onboundary}
\end{align}
Solving \eqref{eq:dpcbf_after_tilde_onboundary} for $\|\vv_{\textup{rel}}\|$ gives \eqref{eq:vrel_exact}. The lower bound~\eqref{eq:vrel_lower}
follows from $d(\vx) \geq d_{\min}$ and the fact that the denominator in \eqref{eq:vrel_exact} is upper-bounded above by $1$.
\end{proof}

\begin{lemma}[Trigonometric Bounds on the Boundary]\label{lemma:trigono_bound}
Every boundary state $\vx \in \partial \mathcal{C}$ obeys the bounds:
\begin{subequations}\label{eq:cos_sin_bounds}
\begin{align}
&-1\leq\cos\tilde\psi\le \cos\tilde\psi_{\max}=-\frac{k_{\mu}d_{\min}}{\|\vv_{\textup{rel}}\|_{\max}}, \label{eq:cos_bound}\\
&0\le|\sin\tilde\psi|\leq \sin\tilde\psi_{\max}=\sqrt{1-\cos^{2}\tilde\psi_{\max}}. \label{eq:sin_bound}
\end{align}
\end{subequations}
The condition $\cos \tilde \psi < 0$ implies that on the safety boundary, the robot and obstacle are always moving towards each other in the LoS frame.
\end{lemma}

\begin{proof}
From the boundary identity~\eqref{eq:dpcbf_after_tilde_onboundary}, since $k_\lambda, k_\mu$, and  $d(\vx)$ are positive, the term $\|\vv_{\textup{rel}}\| \cos \tilde \psi$ must be negative, implying $\cos \tilde \psi < 0$. Solving \eqref{eq:dpcbf_after_tilde_onboundary} for $\cos \tilde \psi$ and maximizing the right-hand side over $d(\vx) \geq d_{\min}$ and $\|\vv_{\textup{rel}}\| \leq \|\vv_{\textup{rel}}\|_{\textup{max}}$ yields the upper bound~\eqref{eq:cos_bound}. The bound~\eqref{eq:sin_bound} follows directly.
\end{proof}

\begin{corollary}[Bounds on line-of-sight Relative Velocity]\label{cor:vrel_tilde_bounds}
Under the hypotheses of \autoref{lemma:trigono_bound},
every boundary state $\vx\in\partial\mathcal C$ satisfies:  
\begin{align} \label{eq:vtil_bounds}
   -\|\vv_{\textup{rel}}\|
       &\leq \tilde v_{\textup{rel},x}
       \leq \|\vv_{\textup{rel}}\| \cos\tilde\psi_{\max}, \\
   0 &\leq |\tilde v_{\textup{rel},y}|
       \le \|\vv_{\textup{rel}}\| \sin\tilde\psi_{\max}.
\end{align}
In particular,
\begin{equation}
       \inf_{\vx\in\partial\mathcal C}\bigl|\tilde v_{\textup{rel},y}(\vx)\bigr|=0.
\end{equation}
\end{corollary}
\subsection{Proof of Validity for Case 1~(Steering-Dominant)}\label{app:case1}

In this section, we verify the sufficient CBF from \eqref{eq:dpcbf_condition_converted} for the first subset of the safety boundary, $\partial \mathcal{C}_{1}$.
\paragraph{Definition of the Subspace}
Case 1 corresponds to the \textbf{\emph{steering-dominant}} scenario, defined by the subspace:

\begin{equation}\label{eq:appendix_case1_subspace}
  \partial\mathcal C_{1}
  =\bigl\{\vx\in\partial\mathcal C \big| \sin\tilde\theta\ge\bar s \bigr\},
  \qquad \bar s \;\coloneqq\; \frac{v_{\textup{obs}}}{v}\sin\tilde\theta_{\textup{obs}} \in [0,1).
\end{equation}
In this configuration, as shown in Fig.~\ref{fig:dpcbf_appendix_proof_partition}a, the robot's heading has a significant component directed towards the obstacle's path, and $\sin \tilde \theta$ is uniformly bounded below by $\bar s$.

\paragraph{Proof strategy}
For every $ \vx\in\partial\mathcal C_{1}$, the robot's forward speed $v$ is positive $(v\ge v_{\min}>0)$ by \autoref{assum:robot_speed_bound}. To simplify the analysis of the control coefficients, we normalize the CBF condition by $v$, which gives the equivalent objective:
\begin{equation}\label{eq:appendix_case1_condition}
  \inf_{\vx\in\partial\mathcal C_{1}} \frac{\lieder_f h(\vx)}{v} +
  \inf_{\vx\in\partial\mathcal C_{1}} \bigg|\frac{C^{a}(\vx)}{v}\bigg| a_{\max} + 
  \inf_{\vx\in\partial\mathcal C_{1}} \bigg|\frac{C^{\beta}(\vx)}{v}\bigg| \beta_{\max}
  \geq 0 .
\end{equation}
Our strategy is to show that in this subspace, steering authority is the dominant term. Specifically, we will prove that:
\begin{enumerate}
    \item The worst-case control authority from acceleration is negligible: $\inf_{\vx \in \partial \mathcal{C}_{1}} |C^{a}(\vx) / v| = 0$.
    \item The control authority from steering is strictly positive: $\inf_{\vx \in \partial \mathcal{C}_{1}} |C^{\beta}(\vx) / v| > 0$.
    \item The positive lower bound on steering authority from (ii) is sufficient to overcome the negative lower bound (i.e., worst-case drift) of the drift term $\lieder_f h(\vx) / v$.
\end{enumerate}

\subsubsection{Acceleration Term $(C^{a})$}
We first establish that the infimum of the normalized acceleration coefficient is zero.
 
\begin{align}
\inf_{\vx \in \partial \calC_{1}}  \bigg|\frac{C^{a}(\vx)}{v}\bigg|
&=\inf_{\vx \in \partial \calC_{1}} \bigg| \underbrace{  \Bigl[
        -\frac{1}{v}
        + k_\lambda \frac{d(\vx)}{\|v_{\textup{rel}}\|^{3}} \frac{v_{\textup{obs}}}{v}\,\cos\tilde\theta_{\textup{obs}}\tilde v_{\textup{rel},y}^{2}
      \Bigr]}_{:=\frac{\eta^{a}_{\cos}(\vx)}{v}}\cos\tilde \theta \nonumber \\
      &+  \underbrace{\Bigl[
        k_\lambda \frac{d(\vx)}{\|v_{\textup{rel}}\|^{3}} \frac{v_{\textup{obs}}}{v}\sin\tilde\theta_{\textup{obs}}\tilde v_{\textup{rel},y}^{2}
          -2k_\lambda \frac{d(\vx)}{\|v_{\textup{rel}}\|} 
            \frac{\tilde v_{\textup{rel},y}}{v}
      \Bigr]}_{:=\frac{\eta^{a}_{\sin}(\vx)}{v}}
      \sin\tilde \theta \nonumber \\
      &+ \underbrace{\Bigl[
        - k_\lambda \frac{d(\vx)}{\|v_{\textup{rel}}\|^{3}}
          \tilde v_{\textup{rel},y}^{2}
      \Bigr]}_{:=\frac{\eta^{a}_{0}(\vx)}{v}} \bigg| \\
    \left\{
\begin{array}{l}
  \text{\autoref{cor:vrel_tilde_bounds}}, \\
    \sin\tilde{\theta} \geq \bar s
\end{array}
\right. \Rightarrow  \quad & \geq
    \inf_{\vx \in \partial \calC_{1}} \biggl| \frac{\eta^{a}_{\cos}(\vx)}{v} \cancelto{0}{\cos\tilde \theta} \quad
    + \cancelto{0}{\frac{\eta^{a}_{\sin}(\vx)}{v}} \sin\tilde{\theta}
    + \cancelto{0}{\frac{\eta^{a}_{0}(\vx)}{v}}
     \quad \biggr| = 0.
\end{align}
This infimum is achieved and is exactly zero. The terms composing $C^{a}(\vx) / v$ are functions of $\tilde v_{\textup{rel}, y}$. From \autoref{cor:vrel_tilde_bounds}, we know that $\inf_{\vx \in \partial \mathcal{C}} |\tilde v_{\textup{rel}, y}|=0$. Since $\partial \mathcal{C}_{1} \subset \partial \mathcal{C}$, there exists a sequence in $\partial \mathcal{C}_{1}$ with $\tilde v_{\textup{rel}, y} \to 0$. In this limit, all terms involving $\tilde v_{\textup{rel}, y}$ vanish. The only remaining term is proportional to $\cos \tilde \theta$, which can also be zero within the set (e.g., at $\tilde \theta = \pi / 2$). Thus, the expression can approach zero, and its infimum is 
\begin{equation}
    \inf_{\vx \in \partial \calC_{1}}  \bigg|\frac{C^{a}(\vx)}{v}\bigg| = 0
\end{equation}

\subsubsection{Steering Term $(C^{\beta})$}
Next, we show that the steering term is uniformly positive. In the subspace $\partial \mathcal{C}_{1}$, the condition $\sin \tilde \theta \geq \bar s$ ensures $\tilde v_{\textup{rel}, y} = -v \sin \tilde \theta + v_{\textup{obs}}\sin \tilde \theta_{\textup{obs}} \leq 0$. This geometric configuration requires a positive (counter-clockwise) steering input $\beta$ to generate lateral clearance. An examination of the terms in $C^{\beta} (\vx) / v$ confirms most of them are non-negative in this subspace. Therefore, we can write
\begin{align}
 \inf_{\vx \in \partial \calC_{1}}  \bigg|\frac{C^{\beta}(\vx)}{v}\bigg| &= \inf_{\vx \in \partial \calC_{1}}\bigg| \underbrace{\Bigl[
         \Bigl(
            k_\lambda \frac{\|p_{\textup{rel}}\|}{d(\vx)}
              \frac{\tilde v_{\textup{rel},y}^{2}}{\|v_{\textup{rel}}\|}
          + k_\mu \frac{\|p_{\textup{rel}}\|}{d(\vx)}\Bigr)
        + \frac{v}{\ell_r}\Bigl(
            1
          - k_\lambda\,\frac{d(\vx)}{\|v_{\textup{rel}}\|^{3}}
        v_{\textup{obs}}\cos \tilde \theta_{\textup{obs}}\tilde v_{\textup{rel},y}^{2}
        \Bigr)
    \Bigr]}_{:=\frac{\eta^{\beta}_{\sin}(\vx)}{v}}
      \sin\tilde \theta \nonumber
  \\
    &+ \underbrace{  \Bigl[
            -\frac{\tilde v_{\textup{rel},y}}{\|p_{\textup{rel}}\|}
            + 2k_\lambda \frac{d(\vx)}{\|v_{\textup{rel}}\|} 
              \frac{\tilde v_{\textup{rel},y}\tilde v_{\textup{rel},x}}{\|p_{\textup{rel}}\|}
        + \frac{v}{\ell_r} \Bigl(k_\lambda\,\frac{d(\vx)}{\|v_{\textup{rel}}\|^3}v_{\textup{obs}}\,\sin\tilde\theta_{\textup{obs}}\,\tilde v_{\textup{rel},y}^{2}
        - 2\,k_\lambda\frac{d(\vx)}{\|v_{\textup{rel}}\|}\tilde v_{\textup{rel},y}
        \Bigr)
    \Big]}_{:=\frac{\eta^{\beta}_{\cos}(\vx)}{v}}
      \cos\tilde \theta \bigg|\\
 \label{eq:inf_c1_beta} &= \inf_{\vx \in \partial \calC_{1}} \bigg| \frac{\eta^{\beta}_{\sin}(\vx)}{v} \sin\tilde \theta + \frac{\eta^{\beta}_{\cos}(\vx)}{v} \cos\tilde \theta \bigg|,
\end{align}
where
\begin{subequations}
\begin{align}
    &\frac{\eta^{\beta}_{\sin}(\vx)}{v} = 
         \Bigl[ \Bigl(
            \underbrace{k_\lambda \frac{\|p_{\textup{rel}}\|}{d(\vx)} 
              \frac{\tilde v_{\textup{rel},y}^{2}}{\|v_{\textup{rel}}\|}}_{\geq 0}
          + \underbrace{k_\mu \frac{\|p_{\textup{rel}}\|}{d(\vx)}\Bigr)}_{> 0}
        + \Bigl(
            \underbrace{\frac{v}{\ell_r}}_{>0}
           \underbrace{-k_\lambda \frac{d(\vx)}{\|v_{\textup{rel}}\|^{3}} \frac{v}{\ell_r} v_{\textup{obs}} \cos\tilde\theta_{\textup{obs}}\tilde v_{\textup{rel},y}^{2}}_{\geq 0}
         \Bigr) \Bigr], \label{eq:appendix_case1_beta_eta_sin}\\
         &\frac{\eta^{\beta}_{\cos}(\vx)}{v} := \Bigl[
            \underbrace{-\frac{\tilde v_{\textup{rel},y}}{\|p_{\textup{rel}}\|}}_{\geq 0} +
            \underbrace{2k_\lambda \frac{d(\vx)}{\|v_{\textup{rel}}\|}
              \frac{\tilde v_{\textup{rel},y}\tilde v_{\textup{rel},x}}{\|p_{\textup{rel}}\|}}_{\geq 0}
        + \frac{v}{\ell_r} \Bigl(k_\lambda \frac{d(\vx)}{\|v_{\textup{rel}}\|^3}v_{\textup{obs}} \sin\tilde\theta_{\textup{obs}} \tilde v_{\textup{rel},y}^{2} 
        \underbrace{- 2k_\lambda\frac{d(\vx)}{\|v_{\textup{rel}}\|}\tilde v_{\textup{rel},y}}_{\geq 0}
        \Bigr)
    \Big].
    \label{eq:appenidx_case1_beta_eta_cos}
\end{align}
\end{subequations}
Then, for \textbf{Case~1}, the worst–case lower bound of the control input steering term is attained at
\begin{align}
    \bigg|\frac{C^{\beta}(\vx)}{v}\bigg|\beta_{\max} = \left( \frac{C^{\beta}(\vx)}{v} \right) \beta_{\max}
    &\geq  \inf_{\vx \in \partial \calC_{1}} \left( \frac{\eta^{\beta}_{\sin}(\vx)}{v} \sin\tilde \theta + \frac{\eta^{\beta}_{\cos}(\vx)}{v} \cos\tilde \theta \right) \beta_{\max} \\
    {\text{\autoref{prop:inf_subadd}}}
    \Rightarrow & \geq \inf_{\vx \in \partial \calC_{1}} \left( \frac{\eta^{\beta}_{\sin}(\vx)}{v} \sin\tilde \theta\right) \beta_{\max} + \inf_{\vx \in \partial \calC_{1}} \left( \frac{\eta^{\beta}_{\cos}(\vx)}{v} \cos\tilde \theta \right) \beta_{\max} \\
    & \coloneqq C^{\beta}_{1,\min} > 0.
\end{align}
To guarantee that $C^{\beta}_{1,\min}>0$ uniformly over~$\vx$, we formulate equation using the parameters
$k_\lambda$ and $k_\mu$ so that the guaranteed minimum of  
$\left(\frac{\eta^{\beta}_{\sin}(\vx)}{v} \sin\tilde \theta + \frac{\eta^{\beta}_{\cos}(\vx)}{v} \cos\tilde \theta \right)$ strictly non-zero.
\allowdisplaybreaks
\paragraph{Bounding the $\sin \tilde \theta$ coefficient}
By \eqref{eq:appendix_case1_beta_eta_sin}, all the terms within $\eta^{\beta}_{\sin} (\vx) / v$ are non-negative. Hence, we can treat the absolute value as a sum of scalar functions.
\begin{align}
\inf_{\vx \in \partial \calC_{1}}\frac{\eta^{\beta}_{\sin}(\vx)}{v} \sin \tilde \theta = & \inf_{\vx \in \partial \calC_{1}}
         \bigg[\bigg(
            k_\lambda\,\frac{\|p_{\textup{rel}}\|}{d(\vx)}\,
              \frac{\tilde v_{\textup{rel},y}^{2}}{\|v_{\textup{rel}}\|}
          + k_\mu\,\frac{\|p_{\textup{rel}}\|}{d(\vx)}\bigg)
        + \bigg(
            \frac{v}{\ell_r}
           -k_\lambda\,\frac{d(\vx)}{\|v_{\textup{rel}}\|^{3}}\,\frac{v}{\ell_r}\,v_{\textup{obs}}\,\cos\tilde\theta_{\textup{obs}}\tilde v_{\textup{rel},y}^{2}
           \bigg)\bigg] \sin \tilde \theta \label{eq:inf_c1_beta_nonabs}\\
        \text{\autoref{prop:inf_subadd}} \Rightarrow  \quad  \geq & \inf_{\vx \in \partial \calC_{1}}
         \bigg(
            k_\lambda\,\frac{\|p_{\textup{rel}}\|}{d(\vx)}
              \frac{\tilde v_{\textup{rel},y}^{2}}{\|v_{\textup{rel}}\|}\bigg) \sin \tilde \theta
          +\inf_{\vx \in \partial \calC_{1}} \bigg(k_\mu\,\frac{\|p_{\textup{rel}}\|}{d(\vx)}\bigg) \sin \tilde \theta \nonumber\\
        & \qquad \qquad \quad +\inf_{\vx \in \partial \calC_{1}}\bigg(\frac{v}{\ell_r}\bigg) \sin \tilde \theta +\inf_{\vx \in \partial \calC_{1}}\bigg(-k_\lambda\,\frac{d(\vx)}{\|v_{\textup{rel}}\|^{3}}\frac{v}{\ell_r}\,v_{\textup{obs}}\cos\tilde\theta_{\textup{obs}}\tilde v_{\textup{rel},y}^{2}\bigg) \sin \tilde \theta \\
        \left\{
        \begin{array}{l}
        \text{\autoref{cor:vrel_tilde_bounds}}, \\
        \sin\tilde{\theta} \geq \bar s
        \end{array}
        \right.
        \Rightarrow  \quad  = &\inf_{\vx \in \partial \calC_{1}}
         \bigg(
            k_\lambda \frac{\|p_{\textup{rel}}\|}{d(\vx)}
              \frac{\cancelto{0}{\tilde v_{\textup{rel},y}^{2}}}{\|v_{\textup{rel}}\|}\bigg) \bar s
          +\inf_{\vx \in \partial \calC_{1}} \bigg(k_\mu\,\frac{\|p_{\textup{rel}}\|}{d(\vx)}\bigg) \bar s \nonumber\\
        &\qquad \qquad \quad +\inf_{\vx \in \partial \calC_{1}}\bigg(\frac{v}{\ell_r}\bigg) \bar s
        +\inf_{\vx \in \partial \calC_{1}}\bigg(-k_\lambda \frac{d(\vx)}{\|v_{\textup{rel}}\|^{3}} \frac{v}{\ell_r}\,v_{\textup{obs}} \cos\tilde\theta_{\textup{obs}}\cancelto{0}{\tilde v_{\textup{rel},y}^{2}}\quad \bigg) \bar s
        \\
        =&\left(k_\mu\frac{p_{\max}}{d_{\max}} + \frac{v_{\min}}{\ell_r}\right) \bar s
        =: \eta^{\beta}_{\sin, \min}(k_\mu) \, \bar s.
        \label{eq:app_eta_beta_sin_min}
\end{align}

\paragraph{Bounding the $\cos \tilde \theta$ coefficient}
Similarly, all terms within $\eta^{\beta}_{\cos} (\vx) / v$ are non-negative. However, they all depend on $\tilde v_{\textup{rel} y}$. Since $\inf_{\vx \in \partial \calC_{1}} |\tilde v_{\textup{rel}, y}|=0$, the worst-case lower bound is zero.
\begin{align}
    \inf_{\vx \in \partial \calC_{1}} \frac{\eta^{\beta}_{\cos}(\vx)}{v} \cos \tilde \theta
    &=  \inf_{\vx \in \partial \calC_{1}} \Bigl[
            -\frac{\tilde v_{\textup{rel},y}}{\|p_{\textup{rel}}\|} +
            2k_\lambda\,\frac{d(\vx)}{\|v_{\textup{rel}}\|}
              \frac{\tilde v_{\textup{rel},y}\tilde v_{\textup{rel},x}}{\|p_{\textup{rel}}\|}
        + \frac{v}{\ell_r}\Bigl(k_\lambda \frac{d(\vx)}{\|v_{\textup{rel}}\|^3}v_{\textup{obs}}\,\sin\tilde\theta_{\textup{obs}} \tilde v_{\textup{rel},y}^{2} - 2k_\lambda\frac{d(\vx)}{\|v_{\textup{rel}}\|}\tilde v_{\textup{rel},y}
        \Bigr)
    \Big] \cos \tilde \theta\\
    \left\{
    \begin{array}{l}
          \text{\autoref{cor:vrel_tilde_bounds}} \\
         \sin \tilde \theta \geq \bar s
    \end{array}
    \right.
     \Rightarrow \quad &=  \inf_{\vx \in \partial \calC_{1}}\Bigl[
            -\frac{\cancelto{0}{\tilde v_{\textup{rel},y}}}{\|p_{\textup{rel}}\|} 
            + 2k_\lambda\,\frac{d(\vx)}{\|v_{\textup{rel}}\|}
              \frac{\cancelto{0}{\tilde v_{\textup{rel},y}}\tilde v_{\textup{rel},x}}{\|p_{\textup{rel}}\|} \nonumber \\
        & \qquad \qquad \qquad \qquad \quad + \frac{v}{\ell_r} \Bigl(k_\lambda\frac{d(\vx)}{\|v_{\textup{rel}}\|^3}v_{\textup{obs}}\sin\tilde\theta_{\textup{obs}}\cancelto{0}{\tilde v_{\textup{rel},y}^{2}}
        - 2k_\lambda\frac{d(\vx)}{\|v_{\textup{rel}}\|}\cancelto{0}{\tilde v_{\textup{rel},y}} \quad
        \Bigr)
    \Big] \cancelto{0}{\cos \tilde \theta}\\
    & = 0  \coloneqq \eta^{\beta}_{\cos, \min}. \label{eq:app_eta_beta_cos_min}
\end{align}

\paragraph{Lower Bound on Control Authority}
By combining these explicit bounds, we can establish the concrete lower bound for $\frac{\Phi_{1}(\vx)}{v}$ in this case. By \eqref{eq:app_eta_beta_sin_min} and \eqref{eq:app_eta_beta_cos_min}, the quantity $\left(\frac{\eta^{\beta}_{\sin}(\vx)}{v} \sin\tilde \theta + \frac{\eta^{\beta}_{\cos}(\vx)}{v} \cos\tilde \theta \right)$ is strictly positive, and the minimum available control authority from steering is therefore: 
\begin{equation}
    \frac{\Phi_{1} (\vx)}{v} = \frac{C^{\beta}(\vx)}{v} \beta_{\max} \geq \left[\eta^{\beta}_{\sin, \min}(k_\mu) \, \bar s
\right] \beta_{\max} := C^{\beta}_{1, \min}(k_\mu) > 0.
\end{equation}
\allowdisplaybreaks
\subsubsection{Analysis of Drift Term}
We now find a lower bound for the normalized drift term $\lieder_f h(\vx) / v$
\begin{align}
     \inf_{\vx \in \partial \calC_{1}} \frac{\lieder_f h(\vx)}{v}
    &= \inf_{\vx \in \partial \calC_{1}} \biggl[ \Bigl(
  \underbrace{-\frac{\tilde v_{\textup{rel},y}}{\|p_{\textup{rel}}\|}}_{\geq 0}
  + \underbrace{2k_\lambda
     \frac{d(\vx)}{\|p_{\textup{rel}}\|}\frac{\tilde v_{\textup{rel},y}}{\|v_{\textup{rel}}\|}\tilde v_{\textup{rel},x}}_{\geq 0} \Bigr)\sin\tilde \theta
  + \Bigl( \underbrace{-k_\lambda
     \frac{\|p_{\textup{rel}}\|}{d(\vx)}
     \frac{\tilde v_{\textup{rel},y}^2}{\|v_{\textup{rel}}\|}}_{\leq 0} 
     \underbrace{- k_\mu
     \frac{\|p_{\textup{rel}}\|}{d(\vx)}}_{\leq 0} \Bigr) \cos\tilde \theta \biggr]
    \\
 \text{\autoref{prop:inf_subadd}} \Rightarrow \quad  & \geq
 \inf_{\vx \in \partial \calC_{1}} \biggl[ \Bigl(
    -\frac{\tilde v_{\textup{rel},y}}{\|p_{\textup{rel}}\|}
  +2k_\lambda
     \frac{d(\vx)}{\|p_{\textup{rel}}\|}\frac{\tilde v_{\textup{rel},y}}{\|v_{\textup{rel}}\|}\tilde v_{\textup{rel},x} \Bigr) \sin\tilde \theta \biggr] \nonumber \\
 &\qquad \qquad \qquad \qquad \qquad \qquad +  \inf_{\vx \in \partial \calC_{1}}\biggl[-\Bigl(k_\lambda
     \frac{\|p_{\textup{rel}}\|}{d(\vx)}
     \frac{\tilde v_{\textup{rel},y}^2}{\|v_{\textup{rel}}\|} 
     + k_\mu
     \frac{\|p_{\textup{rel}}\|}{d(\vx)} \Bigr) \cos\tilde \theta \biggr] \\
    |\sin\tilde\theta|\,\geq\,\bar s \Rightarrow \quad  
    &= \inf_{\vx \in \partial \calC_{1}} \biggl[ \Bigl(
    -\frac{\tilde v_{\textup{rel},y}}{\|p_{\textup{rel}}\|}
  +2k_\lambda
     \frac{d(\vx)}{\|p_{\textup{rel}}\|}\frac{\tilde v_{\textup{rel},y}}{\|v_{\textup{rel}}\|}\tilde v_{\textup{rel},x} \Bigr) \bar s \biggr] \nonumber \\
 &\qquad \qquad \qquad \qquad \qquad \qquad +  \inf_{\vx \in \partial \calC_{1}}\biggl[-\Bigl(k_\lambda
     \frac{\|p_{\textup{rel}}\|}{d(\vx)}
     \frac{\tilde v_{\textup{rel},y}^2}{\|v_{\textup{rel}}\|}
     + k_\mu
     \frac{\|p_{\textup{rel}}\|}{d(\vx)} \Bigr) \sqrt{1-\bar s^{2}} \biggr] \\
     \text{\autoref{cor:vrel_tilde_bounds}} \Rightarrow \quad
     &= \inf_{\vx \in \partial \calC_{1}} \biggl[ \Bigl(
    -\frac{\cancelto{0}{\tilde v_{\textup{rel},y}}}{\|p_{\textup{rel}}\|}
  +2k_\lambda
     \frac{d(\vx)}{\|p_{\textup{rel}}\|}\frac{\cancelto{0}{\tilde v_{\textup{rel},y}}}{\|v_{\textup{rel}}\|}\tilde v_{\textup{rel},x} \Bigr) \bar s \biggr] \nonumber \\
 &\qquad \qquad \qquad \qquad \qquad \qquad +  \inf_{\vx \in \partial \calC_{1}}\biggl[-\Bigl(k_\lambda
     \frac{\|p_{\textup{rel}}\|}{d(\vx)}
     \frac{\tilde v_{\textup{rel},y}^2}{\|v_{\textup{rel}}\|}
     + k_\mu
     \frac{\|p_{\textup{rel}}\|}{d(\vx)} \Bigr) \sqrt{1-\bar s^{2}} \biggr] \\
     \text{\autoref{prop:inf_subadd}} \Rightarrow \quad
     & \geq \sqrt{1-\bar s^{2}} \biggl[\inf_{\vx \in \partial \calC_{1}} \Bigl(-k_\lambda
     \frac{\|p_{\textup{rel}}\|}{d(\vx)}
     \frac{\tilde v_{\textup{rel},y}^2}{\|v_{\textup{rel}}\|} \Bigr)
     + \inf_{\vx \in \partial \calC_{1}} \Bigl(-k_\mu
     \frac{\|p_{\textup{rel}}\|}{d(\vx)} \Bigr) \biggr] \\
     \text{\eqref{eq:tilded_vel}} \Rightarrow \quad
     &= \sqrt{1-\bar s^{2}} \biggl[\inf_{\vx \in \partial \calC_{1}} \Bigl(-k_\lambda
     \frac{\|p_{\textup{rel}}\|}{d(\vx)}
     \|v_{\textup{rel}}\|\,\sin^{2}\tilde\psi \Bigr)
     + \inf_{\vx \in \partial \calC_{1}} \Bigl(-k_\mu
     \frac{\|p_{\textup{rel}}\|}{d(\vx)} \Bigr) \biggr] \\
     \text{\autoref{lemma:trigono_bound}} \Rightarrow \quad
     & = -\sqrt{1-\bar s^{2}} \Bigl( k_\lambda \frac{p_{\max}}{d_{\max}}\|v_{\textup{rel}}\|_{\max}\,\sin^{2}\tilde\psi_{\max} 
     + k_\mu \frac{p_{\max}}{d_{\max}} \Bigr)
     := D_{1, \min} (k_\lambda, k_\mu).
 \label{eq:L1max_bound}
 \end{align}
 
\subsubsection{Final CBF Condition Synthesis for Case 1}
Combining the bounds, the sufficient condition~\eqref{eq:appendix_case1_condition} is satisfied for case~1 if
\begin{equation}
   \dot h (\vx, \vu) \geq
   \underbrace{D_{1,\min}(k_\lambda, k_\mu)}_{ <0}
        + \underbrace{\Phi_{1,\min}(k_\mu)}_{>0}
        \geq 0,
   \quad\forall \vx: \partial \mathcal{C}_{1}.
\end{equation}
This holds if we select positive parameters $k_\lambda, k_\mu$ such that
\begin{equation}\label{eq:case1_final}
       \Phi_{1,\min}(k_\mu) \geq -D_{1,\min}(k_\lambda, k_\mu).
\end{equation}
This verifies the CBF condition on $\partial \mathcal{C}_{1}$ and completes the analysis for Case 1.
\subsection{Proof of Validity for Case 2~(Longitudinal-Dominant)}\label{app:case2}

This section verifies the sufficient CBF condition~\eqref{eq:dpcbf_condition_converted} for the second subset of the safety boundary, $\partial \mathcal{C}_{2}$

\paragraph{Define the subspace} 
Case 2 corresponds to the \textbf{\emph{longitudinal-dominant}} scenario, where the robot's heading is nearly aligned with the line-of-sight to the obstacle. The subspace is defined as
\begin{equation} \label{eq:appendix_case2_subspace}
       \partial\mathcal C_{2} = \{\vx\in\partial\mathcal C \mid \sin\tilde\theta < \bar{s} \},
   \qquad \bar s \coloneqq \frac{v_{\textup{obs}}}{v}\sin \tilde \theta_{\textup{obs}} \in [0,1).
\end{equation}
In this configuration, as shown in Fig.~\ref{fig:dpcbf_appendix_proof_partition}(b), $\cos \tilde \theta$ is uniformly bounded below by $\sqrt{1-\bar s ^{2}}> 0$.

\paragraph{Proof strategy} 
The objective is to prove that
\begin{equation}\label{eq:appendix_case2_goal}
      \inf_{\vx\in\partial\mathcal C_{2}} \lieder_f h(\vx)
      + \inf_{\vx\in\partial\mathcal C_{2}} \bigl|C^{a}(\vx) \bigr| a_{\max}
      + \inf_{\vx\in\partial\mathcal C_{2}} \bigl| C^{\beta}(\vx) \bigr| \beta_{\max} \geq 0 .
\end{equation}
Our strategy parallels that of Case 1 but highlights the dominance of the acceleration input. We will prove: 
\begin{enumerate}
\item The worst-case control authority from steering is negligible: $\inf_{\vx \in \partial \mathcal{C}_{2}} |C^{\beta}(\vx)| = 0$.
\item The control authority from acceleration is strictly positive: $\inf_{\vx \in \partial \mathcal{C}_{2}} |C^{a}(\vx)| > 0$.
\item This positive lower bound on acceleration authority from (ii) is sufficient to overcome the worst-case drift term $\lieder_f h(\vx)$
\end{enumerate}

Under \eqref{eq:appendix_case2_subspace}, we have $v_{\textup{rel},y} > 0$. Geometrically, states in $\partial \mathcal{C}_{2}$ are those in which the longitudinal input $a$ is the dominant action for avoiding collision with a single obstacle. By contrast, the worst-case effect of the steering input on the barrier function is negligible, yielding $\inf_{\vx \in \partial \mathcal{C}_{2}} |C^{\beta}(\vx)| = 0$. We therefore derive a positive lower bound for the acceleration term, show that the steering term can vanish in the worst case, and then verify \eqref{eq:appendix_case2_goal}.

\subsubsection{Steering Term $(C^{\beta})$}
We first establish that the infimum of the steering coefficient is zero.
\begin{align}
    \inf_{\vx \in \partial \calC_{2}}
    \bigl|C^{\beta}(\vx)\bigr| =&  \inf_{\vx \in \partial \calC_{2}}
     \biggl|\underbrace{v\Bigl[
         \Bigl(
            k_\lambda \frac{\|p_{\textup{rel}}\|}{d(\vx)} 
              \frac{\tilde v_{\textup{rel},y}^{2}}{\|v_{\textup{rel}}\|}
          + k_\mu \frac{\|p_{\textup{rel}}\|}{d(\vx)}\Bigr)
        + \frac{v}{\ell_r}\Bigl(
            1
          - k_\lambda \frac{d(\vx)}{\|v_{\textup{rel}}\|^{3}}v_{\textup{obs}}\cos\tilde\theta_{\textup{obs}}\tilde v_{\textup{rel},y}^{2}
        \Bigr)
      \Bigr]}_{:=\eta^{\beta}_{\sin}(\vx)} 
      \sin\tilde \theta \nonumber
  \\
     &+  \underbrace{v\Bigl[
            -\frac{\tilde v_{\textup{rel},y}}{\|p_{\textup{rel}}\|}
            + 2k_\lambda \frac{d(\vx)}{\|v_{\textup{rel}}\|}
              \frac{\tilde v_{\textup{rel},y}\tilde v_{\textup{rel},x}}{\|p_{\textup{rel}}\|}
        + \frac{v}{\ell_r} \Bigl(k_\lambda\,\frac{d(\vx)}{\|v_{\textup{rel}}\|^3} v_{\textup{obs}} \sin\tilde\theta_{\textup{obs}} \tilde v_{\textup{rel},y}^{2}
        - 2 k_\lambda\frac{d(\vx)}{\|v_{\textup{rel}}\|} \tilde v_{\textup{rel},y}
        \big)
      \Big]}_{:=\eta^{\beta}_{\cos}(\vx)}
      \cos\tilde \theta \biggr|\\
      =& \inf_{\vx \in \partial \calC_{2}} \biggl[ v\Bigl[
         \Bigl( \underbrace{
            k_\lambda \frac{\|p_{\textup{rel}}\|}{d(\vx)} 
              \frac{\tilde v_{\textup{rel},y}^{2}}{\|v_{\textup{rel}}\|}}_{> 0}
          + \underbrace{k_\mu \frac{\|p_{\textup{rel}}\|}{d(\vx)}\Bigr)}_{\geq 0}
        + \underbrace{\frac{v}{\ell_r}\Bigl(
            1}_{> 0}
          \underbrace{- k_\lambda \frac{d(\vx)}{\|v_{\textup{rel}}\|^{3}}v_{\textup{obs}} \cos \tilde \theta_{\textup{obs}}\tilde v_{\textup{rel},y}^{2}
        \Bigr)
      \Bigr]}_{> 0} 
      \sin\tilde \theta \nonumber
  \\
    & +  v\Bigl[
            \underbrace{-\frac{\tilde v_{\textup{rel},y}}{\|p_{\textup{rel}}\|}}_{< 0}
            + \underbrace{2k_\lambda \frac{d(\vx)}{\|v_{\textup{rel}}\|}
              \frac{\tilde v_{\textup{rel},y}\tilde v_{\textup{rel},x}}{\|p_{\textup{rel}}\|}}_{< 0}
        + \frac{v}{\ell_r} \Bigl(\underbrace{k_\lambda\,\frac{d(\vx)}{\|v_{\textup{rel}}\|^3} v_{\textup{obs}}\,\sin\tilde\theta_{\textup{obs}} \tilde v_{\textup{rel},y}^{2}}_{> 0}
        \underbrace{-2 k_\lambda\frac{d(\vx)}{\|v_{\textup{rel}}\|} \tilde v_{\textup{rel},y}}_{< 0}
        \big)
      \Big] \cos\tilde \theta \biggr]\\
      \left\{
      \begin{array}{l}
           \text{\autoref{cor:vrel_tilde_bounds}},  \\
           0 \leq \sin\tilde{\theta} < \bar s 
      \end{array}
      \right.
        \Rightarrow  \quad \geq& \inf_{\vx \in \partial \calC_{2}} \left[ \eta^{\beta}_{\sin}(\vx) \cancelto{0}{\sin\tilde \theta} + \cancelto{0}{\eta^{\beta}_{\cos}(\vx)} \cos\tilde \theta \right] = 0
\end{align}
In the subspace $\partial \mathcal{C}_{2}$, every term in the expression for $C^{\beta} (\vx)$ is a function of $\tilde v_{\textup{rel}, y}$ or $\sin \tilde \theta$. As established in \autoref{cor:vrel_tilde_bounds}, the infimum of $|\tilde v_{\textup{rel}, y}|$ over the boundary is zero. Since $\partial \mathcal{C}_{2} \subset \partial \mathcal{C}$, a state can exist in this subspace where $\tilde v_{\textup{rel}, y} \rightarrow 0$. In this limit, every term in $C^{\beta}(\vx)$ vanishes. Therefore, the infimum is zero:
\begin{equation}
    \inf_{\vx \in \partial \calC_{2}} |C^{\beta}(\vx)| = 0
\end{equation}
The lower bound is thus given by:
\begin{equation}
    \Phi_{2}(\vx) \ge \inf_{\vx \in \partial \calC_{2}} |C^{a}(\vx)| a_{\max} := \Phi_{2,\min} > 0.
    \label{eq:phi_min2-inequality}
\end{equation}
We now establish a positive lower bound for $|C^{a}(\vx)|$.

\subsubsection{Acceleration Term $(C^{a})$}
Unlike the steering term, the acceleration coefficient $C^{a} (\vx)$ does not vanish. In the subspace $\partial \mathcal{C}_{2}$, the condition $0 \leq \sin \tilde \theta < \bar s$ ensures that $\tilde v_{\textup{rel}, y} = -v \sin \tilde \theta + v_{\textup{obs}} \sin \tilde \theta_{\textup{obs}} > 0$. This configuration requires a negative (deceleration) input $a$ to manage the relative velocity. An examination of the terms in $C^{a}(\vx)$ reveals that the term $- \cos \tilde \theta$ is dominant. Since $\cos \tilde \theta > \sqrt{1- \bar s^{2}}> 0$ in this subspace, this term provides a non-vanishing negative component, ensuring that $|C^{a}(\vx)$ is bounded away from zero. Thus, the absolute value is redundant, and we can write:
\begin{align}
    \inf_{\vx \in \partial \calC_{2}} \bigl|C^{a}(\vx)\bigr| &=  \inf_{\vx \in \partial \calC_{2}} \bigg| \underbrace{  \Bigl[
        -1
        + k_\lambda \frac{d(\vx)}{\|v_{\textup{rel}}\|^{3}}  v_{\textup{obs}} \cos\tilde\theta_{\textup{obs}} \tilde v_{\textup{rel},y}^{2}
      \Bigr]}_{:=\eta^{a}_{\cos}(\vx)} 
      \cos\tilde \theta +  \underbrace{\Bigl[
        k_\lambda \frac{d(\vx)}{\|v_{\textup{rel}}\|^{3}}\, v_{\textup{obs}} \sin\tilde\theta_{\textup{obs}}\tilde v_{\textup{rel},y}^{2}
          -2k_\lambda \frac{d(\vx)}{\|v_{\textup{rel}}\|}
            \tilde v_{\textup{rel},y}
      \Bigr]}_{:=\eta^{a}_{\sin}(\vx)} 
      \sin\tilde \theta \nonumber \\
      &+ \underbrace{\Bigl[
        - k_\lambda \frac{d(\vx)}{\|v_{\textup{rel}}\|^{3}} 
          v \tilde v_{\textup{rel},y}^{2}
      \Bigr]}_{:=\eta^{a}_{0}(\vx)} \bigg| \\
      &=\inf_{\vx \in \partial \calC_{2}}\bigg[ \bigg| \eta^{a}_{\cos}(\vx)\cos\tilde{\theta} + \eta^{a}_{\sin}(\vx)\sin\tilde{\theta} + \eta^{a}_{0}(\vx) \bigg| \bigg],
\end{align}
where
\begin{subequations}
\begin{align}
        & \eta^{a}_{\cos}(\vx) = -1
        + \underbrace{k_\lambda \frac{d(\vx)}{\|v_{\textup{rel}}\|^{3}}  v_{\textup{obs}} \cos\tilde\theta_{\textup{obs}} \tilde v_{\textup{rel},y}^{2}}_{< 0}, \label{eq:appendix_case2_eta_a_cos}\\
        &\eta^{a}_{\sin}(\vx) = \underbrace{\underbrace{k_\lambda \frac{d(\vx)}{\|v_{\textup{rel}}\|^{3}} v_{\textup{obs}} \sin\tilde\theta_{\textup{obs}}\tilde v_{\textup{rel},y}^{2}}_{> 0}
          \underbrace{-2k_\lambda \frac{d(\vx)}{\|v_{\textup{rel}}\|}
            \tilde v_{\textup{rel},y}}_{< 0}}_{>0 \, or\, <0},\label{eq:appendix_case2_eta_a_sin}\\
        &\eta^{a}_{0}(\vx) = \underbrace{- k_\lambda \frac{d(\vx)}{\|v_{\textup{rel}}\|^{3}} 
          v \tilde v_{\textup{rel},y}^{2}}_{< 0}. \label{eq:appendix_case2_eta_a_0}
\end{align}

\end{subequations}
Then, for \textbf{Case 2}, we find a lower bound for $|C^{a}(\vx)|$ using the reverse triangle inequality~(\autoref{cor:reverse_inf}) and is attained at
\begin{align} \label{eq:appendix_case2_c_a_init}
    \inf_{\vx \in \partial \calC_{2}}|C^{a}(\vx)|a_{\textup{max}} 
    &= \inf_{\vx \in \partial \calC_{2}} \biggl| 
    \eta^{a}_{\cos}(\vx) \cos\tilde{\theta}
    + \eta^{a}_{0}(\vx)
    + \eta^{a}_{\sin}(\vx) \sin\tilde{\theta}
     \biggr| a_{\textup{max}} \\
    \text{\autoref{prop:inf_subadd}} \Rightarrow &\geq
    \inf_{\vx \in \partial \calC_{2}} \biggl|
    \eta^{a}_{\cos}(\vx) \cos\tilde{\theta} +\eta^{a}_{\sin}(\vx) \sin\tilde{\theta}\biggr| a_{\textup{max}}
    +\inf_{\vx \in \partial \calC_{2}} \biggl|\eta^{a}_{0}(\vx) \biggr| a_{\textup{max}}\\
     \text{\autoref{cor:reverse_inf}} \Rightarrow & \geq
     \inf_{\vx \in \partial \calC_{2}} \biggl|
    \eta^{a}_{\cos}(\vx) \cos\tilde{\theta}\biggr| a_{\textup{max}}
    +\inf_{\vx \in \partial \calC_{2}} \biggl|\eta^{a}_{0}(\vx) \biggr| a_{\textup{max}}
    -\sup_{\vx \in \partial \calC_{2}} \biggl|\eta^{a}_{\sin}(\vx) \sin\tilde{\theta}
     \biggr| a_{\textup{max}} \\
     & =
     \inf_{\vx \in \partial \calC_{2}} \biggl[
    -\eta^{a}_{\cos}(\vx) \cos\tilde{\theta}\biggr] a_{\textup{max}}
    +\inf_{\vx \in \partial \calC_{2}} \biggl[-\eta^{a}_{0}(\vx) \biggr] a_{\textup{max}}
    -\sup_{\vx \in \partial \calC_{2}} \biggl[\eta^{a}_{\sin}(\vx) \sin\tilde{\theta}
     \biggr] a_{\textup{max}} \label{eq:appendix_case2_nonzero}\\
    &\coloneqq C^{a}_{2, \min} >0.
\end{align}
To guarantee that $C^{a}_{2, \min} > 0$ uniformly over $\vx$, we formulate equation~\eqref{eq:appendix_case2_nonzero} using the parameters $k_{\lambda}$ and $k_{\mu}$ so that the guaranteed minimum of \eqref{eq:appendix_case2_c_a_init} holds.

\paragraph{Bounding the coefficient $-\eta^{a}_{\cos}(\vx) \cos \tilde \theta$}
By \eqref{eq:appendix_case2_eta_a_cos}, all the terms have the same sign over the domain. Hence, we can treat them as a sum of scalar functions. Here, we can find a non-zero uniform lower bound for $-\eta^{a}_{\cos}(\vx) \cos \tilde \theta$  because the constant term 1 is present and also $1 \geq \cos \tilde \theta > \sqrt{1-\bar s^{2}} $ in this scenario.
\begin{align}
\inf_{\vx \in \partial \calC_{2}} \biggl[-\eta^{a}_{\cos}(\vx) \cos \tilde \theta \biggr]
&= \inf_{\vx \in \partial \calC_{2}} \biggl[1
- k_\lambda \frac{d(\vx)}{\|v_{\textup{rel}}\|^{3}} v_{\textup{obs}}\cos\tilde \theta_{\textup{obs}} \tilde{v}_{\textup{rel},y}^{2} \biggr] \cos \tilde \theta \\
\left\{
\begin{array}{l}
     \text{\autoref{cor:vrel_tilde_bounds}}  \\
     \sin \tilde \theta < \bar s 
\end{array}
\right.
 \Rightarrow  \quad  
& = \inf_{\vx \in \partial \calC_{2}} \bigg[ 1 - k_\lambda \frac{d(\vx)}{\|v_{\textup{rel}}\|^{3}} v_{\textup{obs}}\cos\tilde{\theta}_{\textup{obs}} \cancelto{0}{\tilde{v}^2_{\textup{rel},y}} \quad \bigg] \sqrt{1 - \bar s^{2}} \coloneqq \eta^{a}_{\cos,\min}. \label{eq:appendix_case2_eta_a_cos_result}
\end{align}

\paragraph{Bounding the coefficient $-\eta^{a}_{0}(\vx)$}
Similarly,  $\eta^{a}_{0}(\vx)$ has a fixed sign and enforces deceleration in this case, but its infimum vanishes by \autoref{cor:vrel_tilde_bounds}.
\begin{align}
\inf_{\vx \in \partial \calC_{2}} \biggl[-\eta^{a}_{0}(\vx) \biggr]
&= \inf_{\vx \in \partial \calC_{2}} \bigg[k_\lambda\frac{d(\vx)}{\|v_{\textup{rel}}\|^{3}}v \tilde{v}^2_{\textup{rel},y} \bigg] \\
\text{\autoref{cor:vrel_tilde_bounds}} \Rightarrow  \quad   
& = \inf_{\vx \in \partial \calC_{2}}\bigg[k_\lambda\frac{d(\vx)}{\|v_{\textup{rel}}\|^{3}}v \cancelto{0}{\tilde{v}^2_{\textup{rel},y}} \quad \bigg] = 0 
\coloneqq \eta^{a}_{0, \min}.
\end{align}

\paragraph{Bounding the coefficient $\eta^{a}_{\sin}(\vx) \sin \tilde \theta$}
We next upper-bound $\eta^{a}_{\sin}(\vx) \sin \tilde \theta$ using the triangle inequality. Since this term does not have a fixed sign over the domain, bounding its magnitude is the safest way to control its worst-case contribution.
\begin{align}
\sup_{\vx \in \partial \calC_{2}} \biggl[\eta^{a}_{\sin}(\vx) \sin \tilde \theta \biggr]
&= \sup_{\vx \in \partial \calC_{2}} \biggl[\underbrace{k_\lambda\frac{d(\vx)}{\|v_{\textup{rel}}\|^{3}}v_{\textup{obs}}\sin\tilde{\theta}_{\textup{obs}}\tilde{v}^2_{\textup{rel},y}}_{> 0} \underbrace{- 2k_\lambda \frac{d(\vx)}{\|v_{\textup{rel}}\|} \tilde{v}_{\textup{rel},y}}_{< 0} \biggr] \sin \tilde \theta\\
&\geq \sup_{\vx \in \partial \calC_{2}} \biggl[ k_\lambda\frac{d(\vx)}{\|v_{\textup{rel}}\|^{3}}v_{\textup{obs}}\sin\tilde{\theta}_{\textup{obs}}\tilde{v}^2_{\textup{rel},y} \biggr] \sin \tilde \theta
+ \sup_{\vx \in \partial \calC_{2}} \biggl[ -2k_\lambda \frac{d(\vx)}{\|v_{\textup{rel}}\|} \tilde{v}_{\textup{rel},y} \biggr] \sin \tilde \theta\\
&= \sup_{\vx \in \partial \calC_{2}} \biggl[ k_\lambda\frac{d(\vx)}{\|v_{\textup{rel}}\|^{3}}v_{\textup{obs}}\sin\tilde{\theta}_{\textup{obs}}\tilde{v}^2_{\textup{rel},y} \biggr] \sin \tilde \theta
+ \inf_{\vx \in \partial \calC_{2}} \biggl[2k_\lambda \frac{d(\vx)}{\|v_{\textup{rel}}\|} \tilde{v}_{\textup{rel},y} \biggr] \sin \tilde \theta\\
\left\{
\begin{array}{l}
     \text{\autoref{cor:vrel_tilde_bounds}}  \\
     \sin \tilde \theta < \bar s 
\end{array}
\right. \Rightarrow \quad
&= \sup_{\vx \in \partial \calC_{2}} \biggl[ k_\lambda\frac{d(\vx)}{\|v_{\textup{rel}}\|^{3}}v_{\textup{obs}}\sin\tilde{\theta}_{\textup{obs}}\tilde{v}^2_{\textup{rel},y} \biggr] \bar s
+ \inf_{\vx \in \partial \calC_{2}} \biggl[2 k_\lambda \frac{d(\vx)}{\|v_{\textup{rel}}\|} \cancelto{0}{\tilde{v}_{\textup{rel},y}} \quad \biggr] \cancelto{0}{\sin \tilde \theta}\\
\text{\eqref{eq:tilded_vel}} \Rightarrow \quad &= \sup_{\vx \in \partial \calC_{2}}\biggl[ k_\lambda\frac{d(\vx)}{\|v_{\textup{rel}}\|}v_{\textup{obs}}\sin\tilde{\theta}_{\textup{obs}}\sin^{2}\tilde\psi \biggr] \bar s\label{eq:inf_c2_a_sin}\\
\text{\autoref{lemma:trigono_bound}} \Rightarrow  \quad & = \, k_\lambda\frac{d_{\max}}{\|v_{\textup{rel}}\|_{\min}}v_{\textup{obs},\max}\sin^2\tilde\psi_{\max} \, \bar s
:= \bar s \, \eta^{a}_{\sin, \max}(k_\lambda). \label{eq:appendix_case2_eta_a_sin_result}
\end{align}

\paragraph{Lower Bound on Control Terms}
Combining these bounds into \eqref{eq:appendix_case2_nonzero}, we get a strictly positive lower bound for $\Phi_{2}(\vx)$ in this case. By \eqref{eq:appendix_case2_eta_a_cos_result} and \eqref{eq:appendix_case2_eta_a_sin_result}, $\biggl| 
    \eta^{a}_{\cos}(\vx) \cos\tilde{\theta}
    + \eta^{a}_{0}(\vx)
    + \eta^{a}_{\sin}(\vx) \sin\tilde{\theta}
     \biggr|$ become strictly positive and the minimum available control authority from acceleration is therefore:
\begin{equation}
\Phi_2(\vx) \ge \left[\eta^{a}_{\cos, \min} - \bar{s} \, \eta^{a}_{\sin, \max}(k_\lambda)  \right] a_{\max} := \Phi_{2,\min}(k_\lambda) > 0.\label{eq:phi_2_min}
\end{equation}

\subsubsection{Analysis of Drift Term} All terms composing the drift term $\lieder_f h(\vx)$ for $\vx \in \partial \calC_{2}$ are negative in this subspace.
\begin{align}
 \inf_{\vx \in \partial \calC_{2}} \lieder_f h(\vx) &= \inf_{\vx \in \partial \calC_{2}} \bigg[v\bigg( \underbrace{-\frac{\tilde{v}_{\textup{rel},y}}{\|p_{\textup{rel}}\|}\sin\tilde{\theta}}_{\leq 0} 
 \underbrace{- k_\lambda \frac{\|p_{\textup{rel}}\|}{d(\vx)} \frac{\tilde{v}_{\textup{rel},y}^2}{\|v_{\textup{rel}}\|}\cos\tilde{\theta}}_{< 0} 
 \underbrace{+ 2k_\lambda \frac{d(\vx)}{\|p_{\textup{rel}}\|} \frac{\tilde{v}_{\textup{rel},y}}{\|v_{\textup{rel}}\|} \tilde{v}_{\textup{rel},x}\sin\tilde{\theta}}_{\leq 0} 
 \underbrace{- k_\mu \frac{\|p_{\textup{rel}}\|}{d(\vx)}\cos\tilde{\theta}}_{< 0}
 \bigg)\bigg] \label{eq:Lfh_c2_full} \\
\eqref{eq:tilded_vel} \Rightarrow \quad 
&= \inf_{\vx \in \partial \calC_{2}} \bigg[v\bigg( 
-\frac{\|v_{\textup{rel}}\|}{\|p_{\textup{rel}}\|}\sin\tilde\psi\,\sin\tilde{\theta} 
- k_\lambda \frac{\|p_{\textup{rel}}\|}{d(\vx)} \|v_{\textup{rel}}\|\,\sin^{2}\tilde\psi\,\cos\tilde{\theta} \nonumber \\
&\qquad \qquad \qquad \qquad \qquad \qquad + 2k_\lambda \frac{d(\vx)}{\|p_{\textup{rel}}\|} \|v_{\textup{rel}}\|\,\sin\tilde\psi\,\cos\tilde\psi\,\sin\tilde{\theta} 
- k_\mu \frac{\|p_{\textup{rel}}\|}{d(\vx)}\cos\tilde{\theta} \bigg)\bigg]\label{eq:inf_Lfh_c2_lemma3}\\
  \text{\autoref{prop:trigono_max}} \Rightarrow \quad & \geq \inf_{\vx \in \partial \calC_{2}} \bigg[v\bigg( 
  -\frac{\|v_{\textup{rel}}\|}{\|p_{\textup{rel}}\|}\sin\tilde\psi\,\sin\tilde{\theta} 
  - k_\lambda \frac{\|p_{\textup{rel}}\|}{d(\vx)} \|v_{\textup{rel}}\|\,\sin^{2}\tilde\psi \cos\tilde{\theta} \nonumber \\ 
  & \qquad \qquad \qquad \qquad \qquad \qquad+ k_\lambda \frac{d(\vx)}{\|p_{\textup{rel}}\|} \|v_{\textup{rel}}\|\,\sin\tilde{\theta} 
  - k_\mu \frac{\|p_{\textup{rel}}\|}{d(\vx)}\cos\tilde{\theta} \bigg)\bigg]\label{eq:inf_Lfh_c2_prop4}\\
  \text{\autoref{lemma:trigono_bound}}\Rightarrow \quad & = \inf_{\vx \in \partial \calC_{2}} \bigg[ v \bigg(
  -\frac{\|v_{\textup{rel}}\|}{\|p_{\textup{rel}}\|}\sin\tilde\psi_{\max}\,\sin\tilde{\theta}- k_\lambda \frac{\|p_{\textup{rel}}\|}{d(\vx)} \|v_{\textup{rel}}\|\,\sin^{2}\tilde\psi_{\max}\cos\tilde{\theta} \nonumber \\
  & \qquad \qquad \qquad \qquad \qquad \qquad+ k_\lambda \frac{d(\vx)}{\|p_{\textup{rel}}\|} \|v_{\textup{rel}}\|\sin\tilde{\theta} 
  - k_\mu \frac{\|p_{\textup{rel}}\|}{d(\vx)}\cos\tilde{\theta} \bigg)\bigg]\label{eq:inf_Lfh_c2_prop6}\\
 |\sin\tilde\theta| < \bar s \Rightarrow \quad 
 & = \inf_{\vx \in \partial \calC_{2}} \bigg[v\Bigl(
 -\frac{\|v_{\textup{rel}}\|}{\|p_{\textup{rel}}\|}\sin\tilde\psi_{\max}\,\bar s
 - k_\lambda \frac{\|p_{\textup{rel}}\|}{d(\vx)} \|v_{\textup{rel}}\|\,\sin^{2}\tilde\psi_{\max} \sqrt{1-\bar s^{2}} \nonumber \\
 & \qquad \qquad \qquad \qquad \qquad \qquad+ k_\lambda \frac{d(\vx)}{\|p_{\textup{rel}}\|} \|v_{\textup{rel}}\|\,\bar s 
 - k_\mu \frac{\|p_{\textup{rel}}\|}{d(\vx)} \sqrt{1-\bar s^{2}}
 \Bigr)\bigg]\label{eq:inf_Lfh_c2_bars} \\
 \text{\autoref{prop:inf_subadd}} \Rightarrow \quad & \geq -v_{\max}\left( \frac{\|v_{\textup{rel}}\|_{\max}}{p_{\min}}\sin\tilde\psi_{\max}\,\bar{s} + k_\lambda \frac{p_{\max}}{d_{\max}}\|v_{\textup{rel}}\|_{\max}\sin^{2}\tilde\psi_{\max} + k_\lambda \frac{d_{\min}}{p_{\min}}\|v_{\textup{rel}}\|_{\max}\bar{s} + k_\mu \frac{p_{\max}}{d_{\max}} \right)\nonumber\\
 &:= D_{2,\min}(k_\lambda, k_\mu).\label{eq:L2max_bound}
\end{align}

\subsubsection{Final CBF Condition Synthesis for Case 2}
Combining the bounds, the sufficient condition~\eqref{eq:appendix_case2_goal} is satisfied for case~2 if
\begin{equation}
    \dot h (\vx, \vu) \ge 
    \underbrace{D_{2,\min}(k_\lambda, k_\mu)}_{<0} 
    + \underbrace{\Phi_{2,\min}(k_\lambda)}_{>0},
    \quad\forall \vx : \partial \mathcal{C}_{2}.
\end{equation}
This holds if we select positive parameters $k_\lambda, k_\mu$ such that
\begin{equation}\label{eq:case2_final}
       \Phi_{2,\min}(k_\lambda) \geq -D_{2,\min}(k_\lambda, k_\mu)
\end{equation}
This verifies the CBF condition on $\partial \mathcal{C}_{2}$ and completes the analysis for Case 2.
\subsection{Results}\label{app:result}

The analyses in \autoref{app:case1} and \autoref{app:case2} establish that the DPCBF is valid if there exists a pair of positive parameters $(k_\lambda, k_\mu)$ that simultaneously satisfies the final conditions for the steering-dominant and longitudinal-dominant cases derived in \eqref{eq:case1_final} and \eqref{eq:case2_final}, respectively:
\begin{align}
       \Phi_{1,\min}(k_\mu) &\geq -D_{1,\min}(k_\lambda, k_\mu),\\
       \Phi_{2,\min}(k_\lambda) &\geq -D_{2,\min}(k_\lambda, k_\mu).
\end{align}

While these inequalities are complex, they can be evaluated numerically for a given set of system parameters to find a non-empty feasible set for $(k_\lambda, k_\mu)$. This section demonstrates this process, thereby completing the proof of DPCBF validity.

\paragraph{Parameter Evaluation}
We evaluate the bounds derived for $\Phi_{1, \min}, D_{1, \min}, \Phi_{2, \min}$ and $D_{2, \min}$ using the physical parameters of the robot and obstacles from our simulation studies, as summarized in Table I. For the partitioning threshold, a value of $\bar s=0.44$ was selected to ensure a balanced analysis between the two cases for our experiments.

\begin{figure}[h!]
    \centering
    \includegraphics[width=.5\linewidth]{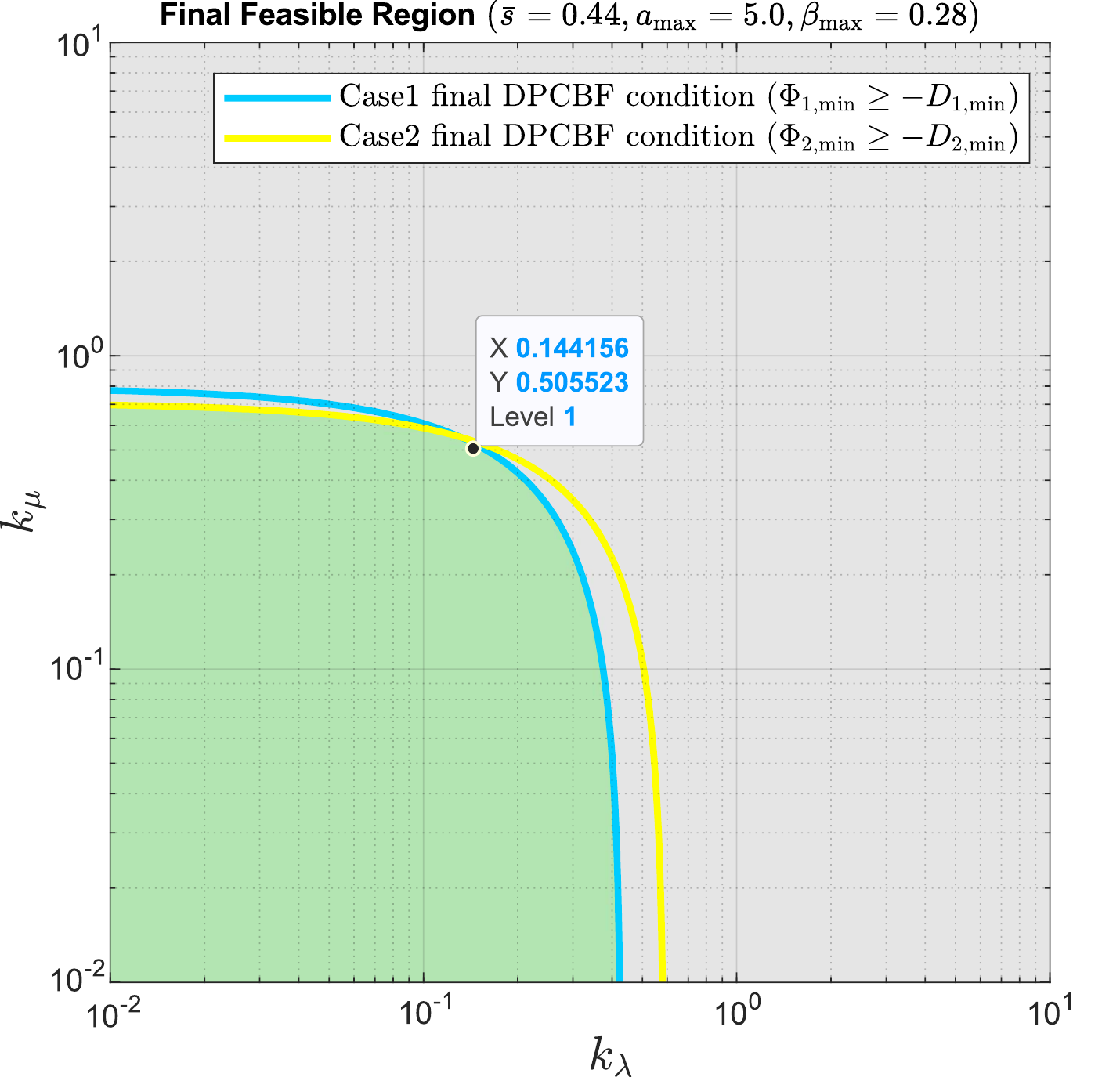}
    \caption{The feasible region for the DPCBF parameters $(k_\lambda, k_\mu)$, evaluated using the system parameters in Table I and a threshold of $\bar s=v_{\textup{obs}, \max} / v_{\max} \approx 0.44$. The green shaded area represents the intersection of the feasible sets for Case 1 (bounded by cyan) and Case 2 (bounded by yellow). The black dot indicates the parameter choice ($k_\lambda=0.144$ and $k_\mu=0.505$) used in our simulations, which lies safety within the proven feasible region with respect to a single obstacle.}
    \label{fig:dpcbf_parameters_feasible_set}
\end{figure}
\paragraph{Feasible Region}
\autoref{fig:dpcbf_parameters_feasible_set} plots the resulting feasible regions for the parameters $(k_\lambda, k_\mu)$ on a log-log scale.
\begin{itemize}
\item The region bounded by the \textbf{cyan line} represents the set of gains satisfying the Case 1 (steering-dominant) condition.
\item The region bounded by the \textbf{yellow line} represents the set of gains satisfying the Case 2 (longitudinal-dominant) condition.
\end{itemize}
The intersection of these two sets, shown as the \textbf{green shaded region}, constitutes the final feasible region. Any pair $(k_\lambda, k_\mu)$ chosen from this region guarantees that the DPCBF is a valid CBF for our system under the specified parameters.

\paragraph{Conclusion}
The existence of this non-empty feasible region completes our proof. For the simulation results presented in the main paper, we selected the tunable parameters $k_\lambda=0.144$ and $k_\mu =0.505$. As shown by the black dot in Fig.~\ref{fig:dpcbf_parameters_feasible_set}, this choice lies within the proven feasible region. This demonstrates that the performance of our DPCBF-based controller observed in simulations is underpinned by this formal guarantee of safety.

\end{document}